\documentclass{article}

\usepackage{jmlr2arxiv}

\usepackage{amsmath}

\usepackage{amsfonts}
\usepackage{subfigure}
\usepackage[boxed,noend]{algorithm2e}
\usepackage{xspace}
\usepackage{ifthen}
\usepackage{microtype}
\usepackage{times}
\usepackage{color}

\usepackage{bbm}
\usepackage{stmaryrd}
\usepackage{varioref}

\usepackage{xtab}

\usepackage[linkbordercolor={0 1 0},pdftex]{hyperref}

\jmlrheading{Volume??}{Year??}{pages??}{??/??}{??/??}{Daniel Golovin, 
  Andreas Krause, and Matthew Streeter}

\newcommand{\ourtitle}{Online Submodular Maximization under a Matroid Constraint with Application to Learning Assignments}

\ShortHeadings{\ourtitle}{Golovin, Krause, and Streeter}
\firstpageno{1}

\newboolean{istechrpt}
\setboolean{istechrpt}{true}  %
\newboolean{isArxiv}          %
\setboolean{isArxiv}{true}   %
\newboolean{showcomments}
\setboolean{showcomments}{false}

\newcommand{\daniel}[1]{\ifthenelse{\boolean{showcomments}}{\textcolor{blue}{||
      Daniel:
      {#1} ||}}{}}

\newcommand{\matt}[1]{\ifthenelse{\boolean{showcomments}}{\textcolor{red}{||
      Matt:
      {#1} ||}}{}}
\newcommand{\andreas}[1]{\ifthenelse{\boolean{showcomments}}{\textcolor{green}{||
      Andreas:
      {#1} ||}}{}}

\newcommand{\commentout}[1]{}

\newcommand{\cV}{{\mathcal{V}}}

\newlength{\presec}
\newlength{\postsec}
\newlength{\presubsec}
\newlength{\postsubsec}
\newlength{\prepara}
\newlength{\postpara}
\newcommand{\initOneLiners}{%
    \setlength{\itemsep}{0pt}
    \setlength{\parsep }{0pt}
    \setlength{\topsep }{0pt}
}

\newcommand{\ignore}[1]{}

\def \argmax {\mathop{\rm arg\,max}}

\newcommand{\paren} [1] {\ensuremath{ \left( {#1} \right) }}

\renewcommand{\Pr}[1]{\ensuremath{\mathbb{P}\left[#1\right] }}
\newcommand{\E}[1]{\ensuremath{\mathbb{E}\left[#1\right] }}
\newcommand{\size}[1]{\ensuremath{\left|#1\right|}}

\newcommand{\tuple}[1]{\ensuremath{\langle #1 \rangle}}

\newcommand{\nats}{\ensuremath{\mathbb{N}}}
\newcommand{\reals}{\ensuremath{\mathbb{R}}}

\newcommand{\bigO}[1]{\ensuremath{O\paren{#1}}}

\newcommand{\class} [1] {\textrm{#1}} %
\renewcommand{\P} {\class{P}}
\newcommand{\NP} {\class{NP}}

\newcommand{\compactdispmath}[1]{\[ \vspace{-1mm} {#1} \vspace{-0mm}  \]}

\newcommand{\roundindex}{\ensuremath{t}}
\newcommand{\partitionindex}{\ensuremath{k}}

\newcommand{\fullversion}{extended version\xspace}

\newcommand{\features}{\ensuremath{v}}
\newcommand{\NumFeatures}{\ensuremath{m}}
\newcommand{\NumRounds}{\ensuremath{T}}

\newcommand{\colorindex}{\ensuremath{c}}
\newcommand{\Color}{\ensuremath{c}}
\newcommand{\Colors}{\ensuremath{\bracket{\NumColors}}}
\newcommand{\Colorvec}{\ensuremath{ \vec{c} }}

\newcommand{\nsamp}{\ensuremath{\rho}}
\newcommand{\est}{\ensuremath{\omega}}

\def \k {\partitionindex}
\def \c {\colorindex}
\def \cvec {\Colorvec}
\def \j {\c}

\def \C {\NumColors}
\def \J {\C}
\def \K {\NumPartitions}
\def \P {\Partition}

\newcommand{\assignment}{S}  %
\newcommand {\groundset}{\ensuremath{\cV}} %

\newcommand {\matroidgroundset}{\groundset}

\newcommand{\OPT}{\textsf{OPT}}

\newcommand{\polytope}[1]{\ensuremath{P\paren{#1}}}

\newcommand{\matroid}{\ensuremath{\mathcal{M}}}
\newcommand{\indset}{\ensuremath{\mathcal{I}}}

\newcommand {\ctrs} {click-through-rates\xspace}

\newcommand{\thresh}{\theta}

\newcommand{\omitproof}[1]{}

\newcommand{\bracket}[1]{\left[#1\right]}
\newcommand{\Esub}[2]{\mathbb{E}_{#1}\bracket{#2}}
\newcommand{\sample}{\textsf{sample}}

\newcommand{\Alg}{\ensuremath{\mathcal{E}}}
\newcommand{\AllRows}{\ensuremath{\mathcal{R}_{\bracket{\J}}}}

\newcommand {\Feasible} {\mathcal{P}}

\newcommand{\NonNegativeReals}{\ensuremath{\mathbb{R}_{\ge 0}}}

\newcommand{\NumColors}{\ensuremath{C}}
\newcommand{\NumPartitions}{\ensuremath{K}}
\newcommand{\OfflineGreedy}{{\sc TabularGreedy}\xspace}
\newcommand{\OnlineGreedy}{{\sc TGonline}\xspace}

\newcommand{\OCG}{{\sc OCG}\xspace}
\newcommand{\ContinuousGreedy}{{\sc ContinuousGreedy}\xspace}
\newcommand{\OnlineContinuousGreedy}{{\sc OnlineContinuousGreedy}\xspace}
\newcommand{\Partition}{\ensuremath{P}}
\newcommand{\Regret}{\ensuremath{r}}
\newcommand{\Row}{\ensuremath{R}}
\newcommand{\Rows}{\ensuremath{\mathcal{R}}}

\newcommand{\conv}{\operatorname{conv}}
\newcommand{\charvec}[1]{\ensuremath{\mathbf{1}_{#1}}}
\newcommand{\vondrak}{Vondr\'{a}k\xspace}

\newtheorem{thm}{Theorem}[section]
\newtheorem{cor}[thm]{Corollary}

\newcommand{\expt}[1]{\mathbb{E}\left[#1\right]}
\newcommand{\expct}[1]{\mathbb{E}\left[#1\right]}

\newcommand{\set}[1]{\left\{#1\right\}}
\newcommand{\Real}{\mathbb R}
\newcommand{\eps}{\varepsilon}

\newcommand{\err}{\epsilon}

\newcommand{\denselist}{
    \itemsep -2pt\topsep-8pt\partopsep-8pt
}

\newcommand{\figref}[1]{Fig.~\ref{#1}}
\newcommand{\eqnref}[1]{Eq.~(\ref{#1})}
\newcommand{\secref}[1]{\S\ref{#1}}
\newcommand{\thmref}[1]{Theorem~\ref{#1}}

\newcommand{\lemref}[1]{Lemma~\ref{#1}}

\newcommand{\probref}[1]{Problem~(\ref{#1})}

\begin{document}

  \title{\ourtitle}
 \author{
     \name \!Daniel Golovin  
     \email {golovin@gmail.com}\\
     \addr \,California Institute of Technology \footnotemark[1]
     \AND
     \name Andreas Krause 
     \email {krausea@ethz.ch}\\                    
     \addr ETH Zurich
    \AND
     \name Matthew Streeter 
     \email {matt@duolingo.com}\\
     \addr Duolingo 
    }

\editor{???}

\maketitle

\renewcommand{\thefootnote}{\fnsymbol{footnote}}

\footnotetext[1]{Work done while at Carnegie Mellon University and at
  the California Institute of Technology.  Current affiliation: Google, Inc. } 
\footnotetext[2]{An earlier version of this
  work appeared as~\cite{streeter09}.}

\begin{abstract}%
Which ads should we display in sponsored search in order to maximize our revenue?
How should we dynamically rank information sources to maximize the
value of the ranking?
These applications exhibit strong diminishing returns: Redundancy
decreases the marginal utility of each ad or information source.
We show that these and other problems can be formalized as
repeatedly selecting 
an assignment of items to positions to maximize a
sequence of \emph{monotone submodular} functions that arrive one by
one.
We present an efficient algorithm
for this general problem and analyze it in the \emph{no-regret} model.
Our algorithm possesses strong theoretical guarantees, such as a
performance ratio that converges to the optimal constant of
$1-1/e$. We empirically evaluate our algorithm on two real-world
online optimization problems on the web: ad allocation with submodular
utilities, and dynamically ranking blogs to detect information
cascades.  Finally, we present a second algorithm that 
handles the more general case in
which the feasible sets are given by a \emph{matroid} constraint, while
still maintaining a $1-1/e$ asymptotic performance ratio.

\end{abstract}


\vspace{\presec}
\section{Introduction} \label{sec:intro}
\vspace{\postsec}

Consider the problem of repeatedly choosing advertisements to display in sponsored search to maximize our revenue.  In this problem, there is a small set of positions on the page, and each time a query arrives we would like to assign,
to each position,
one out of a large number of possible ads.  In this and related problems that we call \emph{online assignment learning} problems,
there is a set of positions, a set of items, and a sequence of rounds, and on each round we must assign an item to each position.  After each round, we obtain some reward
depending on the selected assignment, and we observe the value of the reward.  When there is only one position, this problem
becomes the well-studied multi-armed bandit problem~\citep{auer2002}.
When the positions have a linear ordering
the assignment can be construed as a ranked list of elements, and the
problem becomes one of selecting lists online.
Online assignment learning thus models a central challenge in web
search, sponsored search, news aggregators, and recommendation
systems, among other applications.

A common assumption made in previous work on these problems is that the quality of an assignment is the sum of a function on the (item, position) pairs in the assignment.  For example, online advertising models with \emph {click-through-rates}~\citep{edelman07a} make an assumption of this form.
More recently, there have been attempts to incorporate
the value of diversity in the reward function~\citep{radlinski08}.
Intuitively, even though the best $\NumPartitions$ results for the
query ``turkey'' might happen to be about the country, the best list
of $\NumPartitions$ results is likely to contain some recipes for the
bird as well.  This will be the case if there are diminishing returns
on the number of relevant links presented to a user;
for example, if it is better to present each user with at least one relevant result
than to present half of the users with no relevant results and half with
two relevant results.
We incorporate these considerations in a flexible way by providing an algorithm that performs well whenever the reward for an assignment is a \emph{monotone submodular
function} of its set of (item, position) pairs.
The (simpler) \emph{offline} problem of maximizing a single assignment
for a fixed submodular function is an important special case of the
problem of maximizing a submodular function subject to a
\emph{matroid} constraint. Matroids are important objects in
combinatorial optimization, generalizing the notion of linear
independence in vector spaces. The problem of maximizing a submodular
function subject to arbitrary matroid constraints was recently
resolved by \citet{calinescu07}, who provided an algorithm with
optimal (under reasonable complexity assumptions) approximation
guarantees. In this paper, besides developing a specialized algorithm,
\OnlineGreedy, optimized for online learning of assignments, we also
develop an algorithm, \OCG, for the general problem of maximizing
submodular functions subject to a matroid constraint.  \andreas{Need
  to explain why \OnlineGreedy is a separate algorithm.} \daniel{Done.} 

Our key contributions are:
\begin{enumerate}\denselist
\item[\emph{i)}] an efficient algorithm, \OfflineGreedy, that provides a $(1-1/e)$ approximation ratio for the problem of optimizing assignments under submodular utility functions,
\item[\emph{ii)}] a specialized algorithm for online learning of assignments, \OnlineGreedy, that has strong performance guarantees in the \emph {no-regret} model, 
\item[\emph{iii)}] an algorithm for online maximization of submodular functions subject to a general matroid constraint, 
\item[\emph{iv)}] an empirical evaluation on two problems of information gathering on the web.
\end{enumerate}

This manuscript is organized as follows. In \secref{sec:problem}, we introduce the assignment learning problem. In \secref{sec:algorithm}, we develop a novel algorithm for the \emph{offline} problem of finding a near-optimal assignment subject to a monotone submodular utility function. In \secref{sec:online}, we develop \OnlineGreedy for online learning of assignments. We then, in \secref{sec:continuous-greedy}, develop an algorithm, \OCG, for online optimization of submodular functions subject to arbitrary matroid constraints. We evaluate our algorithms in \secref{sec:evaluation}, review related work in \secref{sec:related_work}, and conclude in \secref{sec:conclusions}.

\vspace{\presec}
\section{The assignment learning problem} %
\label{sec:problem}
\vspace{\postsec}

We consider problems, where we have $\NumPartitions$ positions (e.g., slots for displaying ads), and need to assign to each position an item (e.g., an ad) in order to maximize a utility function (e.g., the revenue from clicks on the ads).  We address both the \emph{offline} problem, where the utility function is specified in advance, and the \emph{online} problem, where a sequence of utility functions arrives over time, and we need to repeatedly select a new assignment in order to maximize the cumulative utility.

\vspace{\prepara}
\paragraph{The Offline Problem.}
In the offline problem we are given sets $\Partition_1$, $\Partition_2$, \ldots, $\Partition_\NumPartitions$, where $\Partition_\k$ is the set of items that may be placed
in position $\k$.  We assume without loss of generality that these sets are disjoint\footnote{If the same item can be placed in multiple positions, simply create multiple distinct copies of it.}. An \emph{assignment} is a subset $\assignment\subseteq\groundset$, where $\groundset=\Partition_1\cup\Partition_2 \cup \dots\cup\Partition_\NumPartitions$ is the set of all items.
We call an assignment \emph{feasible}, if at most one item is assigned to each position (i.e., for all  $\k$, $|\assignment \cap \Partition_\k| \le 1$).  We use $\Feasible$ to refer to the set of feasible assignments.

Our goal is to find a feasible assignment maximizing a utility function $f:2^{\groundset}\to\NonNegativeReals$.
As we discuss later, many important assignment problems satisfy \emph{submodularity}, a natural diminishing returns property:  Assigning a new item to a position $\k$ increases the utility more if few elements have been assigned, and less if many items have already been assigned.  Formally, a utility function $f$ is called submodular, if for all $\assignment\subseteq\assignment'$ and $s\notin\assignment'$ it holds that
$f(\assignment\cup\set{s})-f(\assignment)\geq f(\assignment'\cup\set{s})-f(\assignment')$. We will also assume $f$ is monotone, i.e., for all $\assignment\subseteq\assignment'$, we have $f(\assignment)\leq f(\assignment')$.  Our goal is thus, for a given non-negative, monotone and submodular utility function $f$, to find a feasible assignment $\assignment^*$ of maximum utility, 
$$\assignment^*=\argmax_{\assignment\in\Feasible} f(\assignment).$$
This optimization problem is NP-hard. In fact, \citet{mirrokni08} show that any algorithm that is guaranteed to obtain a solution within a factor of $(1-1/e+\epsilon)$ of the optimal value requires exponentially many evaluations of $f$ in the worst case.
In light of this negative result, we can only hope to efficiently obtain a solution that achieves a fraction of $(1-1/e)$ of the optimal value. In \secref{sec:offline} we develop such an algorithm.

\vspace{\prepara}
\paragraph{The Online Problem.}
The offline problem is inappropriate to model dynamic settings, where
the utility function may change over time, and we need to repeatedly
select new assignments, trading off exploration (experimenting with the
ads displayed to gain information about the utility function), and
exploitation (displaying ads which we believe will maximize utility).
More formally, we face a sequential decision problem, where,
on each round (which, e.g., corresponds to a user query for a
particular term), we want to
select an assignment $\assignment_{\roundindex}$ (ads to display).  We
assume that the sets $\Partition_1$, $\Partition_2$, \ldots,
$\Partition_\NumPartitions$ are fixed in advance for all
rounds.
After we select
the assignment we obtain reward
$f_{\roundindex}(\assignment_{\roundindex})$ for some non-negative monotone
submodular utility function $f_{\roundindex}$.  We call the setting
where we do not get any information about $f_{\roundindex}$ beyond the
reward the \emph{bandit feedback} model.
In contrast, in the \emph{full-information feedback} model we obtain
oracle access to $f_{\roundindex}$ (i.e., we can evaluate
$f_{\roundindex}$ on arbitrary feasible assignments).
Both models arise in real applications, as we show in \secref{sec:evaluation}.

\ignore{
\begin{enumerate}
\item \emph{Constrained Full Information.}
  We obtain
  oracle access to $f_{\roundindex}$, meaning we may obtain
  $f_{\roundindex}(\assignment)$ for any set $\assignment \in \Feasible$ (but not for $\assignment
  \notin \Feasible$).
\item \emph{Partial Information.} %
  We see the values
  $\set{f_{\roundindex}(\assignment_{\roundindex}^{\k}) :
    \k \in \bracket{\NumPartitions} }$, where
$\assignment_{\roundindex}^{\k} :=
\bigcup_{\k' \in \bracket{\k}} \paren{
  \assignment_{\roundindex} \cap P_{\k'} }$.
\item \emph{Bandit Feedback.} We get no additional information beyond $f_{\roundindex}(\assignment_{\roundindex})$.
\end{enumerate}
} %

The goal is to maximize the total reward we obtain, namely
$\sum_{\roundindex} f_{\roundindex}(\assignment_{\roundindex})$.
Following the multi-armed bandit
literature, we evaluate our performance after $\NumRounds$ rounds by comparing our
total reward against that obtained by a clairvoyant algorithm with
knowledge of the sequence of functions $\tuple{f_1, \ldots, f_{\NumRounds}}$,
but with the restriction that it must select the same assignment
on each round.
The difference between the clairvoyant algorithm's total reward and
ours is called our \emph{regret}.
The goal is then to develop an algorithm whose
expected regret grows sublinearly in the number of rounds; such an algorithm is
said to have (or be) \emph{no-regret}.
However, since sums of submodular functions remain submodular, the clairvoyant algorithm has to solve an offline assignment problem with $f(\assignment)=\sum_{\roundindex} f_{\roundindex}(\assignment)$. Considering the hardness of this problem, no polynomial-time algorithm can possibly hope to
achieve a no-regret guarantee. %
To accommodate this fact, we discount the reward of the clairvoyant algorithm by a factor of
$(1-1/e)$: We define the
\emph{$(1-1/e)$-regret} of a random sequence
  $\tuple{\assignment_1, \ldots, \assignment_{\NumRounds}}$ %
  as
\[
\paren{1-\frac{1}{e}} \cdot \max_{\assignment \in \Feasible} \set { \sum_{\roundindex =
  1}^{\NumRounds} f_{\roundindex}(\assignment) }
\ - \
\expt{\sum_{\roundindex = 1}^{\NumRounds} f_{\roundindex}(\assignment_{\roundindex}) } \mbox { .}
\]
\ignore{
Incorporating features, we similarly define the \emph{feature-weighted}
\emph{$(1-1/e)$-regret} as
\[
\max_{\text{feature }q \in \bracket{\NumFeatures}} \left\{ \paren{1-\frac{1}{e}} \cdot \max_{\assignment \in \Feasible} \sum_{\roundindex =
  1}^{\NumRounds} \features^{\roundindex}_{q} \cdot f_{\roundindex}(\assignment)
\quad - \quad
\expt{\sum_{\roundindex = 1}^{\NumRounds} \features^{\roundindex}_{q} \cdot f_{\roundindex}(\assignment_{\roundindex})}    \right\}
\]
} %
Our goal is then to develop efficient algorithms whose
$(1-1/e)$-regret grows sublinearly in $\NumRounds$.

\vspace{\prepara}
\paragraph{Subsumed Models.}
Our model generalizes several common models for sponsored search ad
selection, and web search results.
These include models with \emph{click-through-rates}, in which
it is assumed that each (ad, position) pair has some probability
$p(a, \k)$ of being clicked on, and there is some monetary
reward $b(a)$ that is obtained whenever ad $a$ is clicked on.
Often, the \ctrs are assumed to be \emph{separable},
meaning $p(a, \k)$ has the functional form $\alpha(a) \cdot
\beta(\k)$ for some functions $\alpha$ and $\beta$.
See~\cite{feldman08} and~\cite{lahaie07} for more details on
sponsored search ad allocation.  Note that in both of these cases, the
(expected) reward of a set $S$ of (ad, position) pairs
is $\sum_{(a,\k) \in S}{g(a,\k)}$ for some nonnegative function $g$.
It is easy to verify that such a reward function is monotone submodular.
Thus, we can capture this model in our framework by setting $\Partition_\k = \mathcal{A} \times \set {\k}$, where $\mathcal{A}$ is the set of ads.
Another subsumed model, for web search, appears in~\cite{radlinski08};
it assumes that each user is interested in a particular set of
results, and any list of results that intersects this set generates a unit
of value; all other lists generate no value, and the ordering of results is irrelevant.  Again, the reward function is monotone submodular.
In this setting, it is desirable to display a diverse set of results in order to maximize the likelihood that at least one of them will interest the user.

Our model is flexible in that we can handle position-dependent effects and diversity considerations simultaneously.  For example, we can
handle the case that each user $u$ is interested in a particular set $A_u$ of ads
and looks at a set $I_u$ of positions, and the reward of an assignment $\assignment$ is
any monotone-increasing concave function $g$ of $|\assignment \cap (A_u \times I_u)|$.  If $I_u = \set {1, 2, \ldots, k}$ and $g(x) = x$, this models the case where the quality is the number of relevant result
that appear in the first $k$ positions.
If $I_u$ equals all positions and $g(x) = \min \set{x, 1}$ we recover the model of~\citet{radlinski08}.

\section{An approximation algorithm for the offline problem}

\label {sec:algorithm}
\subsection{The locally greedy algorithm} \label {sec:locally_greedy}
\vspace{\postsubsec}
A simple approach to the assignment problem is the following greedy
procedure: the algorithm steps through all $\K$ positions (according
to some fixed, arbitrary ordering). For position $\k$, it simply
chooses the item that increases the total value as much as possible,
i.e., it chooses
\compactdispmath{s_\k= \argmax_{s \in \Partition_\k} \set { f(\{s_1,\dots,s_{\k-1} \} + s) }\mbox{,}}
where, for a set $S$ and element $e$, we write $S + e$ for $S \cup \set{e}$.
Perhaps surprisingly, no matter which ordering over the positions is chosen, this so-called \emph {locally greedy algorithm} produces an assignment that obtains at least half the optimal value~\citep{fisher78}.  In fact, the following more general result holds.  We will use this lemma in the analysis of our improved offline algorithm, which uses the locally greedy algorithm as a subroutine.

\begin {lemma} \label {lem:locally_greedy}
Suppose $f : 2^\groundset \to \NonNegativeReals$ is of the form
$f(\assignment) = f_0(\assignment) + \sum_{\k=1}^\K f_\k(\assignment
\cap P_\k)$
where $f_0 : 2^\groundset \to \NonNegativeReals$ is monotone submodular, and
$f_\k : 2^{P_\k} \to \NonNegativeReals$ is arbitrary for $\k \ge 1$.
Let $L$ be the solution returned by the locally greedy algorithm.  Then
$$f(L) + f_0(L) \ge \max_{\assignment \in \Feasible}  f(\assignment) .$$
\end {lemma}
\ifthenelse{\boolean{istechrpt}}
{The proof is given in Appendix A.}
{The proof is given in an \fullversion of this paper~\citep{arxiv-version}.}
Observe that in the special case where $f_\k \equiv 0$ for all $\k \ge 1$, Lemma \ref
{lem:locally_greedy} says that  $f(L) \ge \frac 1 2 \max_{\assignment
  \in \Feasible} f(\assignment)$.
\ifthenelse{\not\boolean{istechrpt}}
{In~\citet{arxiv-version} we provide a simple example showing that this $1/2$ approximation ratio is tight.}
{ %

The following example shows that the $1/2$ approximation ratio is
tight. Consider an instance of the ad allocation problem with two ads,
two positions and two users, Alice and Bob. Alice is interested in ad
1, but pays relatively little attention to ads: She will only click on the ad if it appears in the first position. Bob is interested in ad 2, and will look through all positions. Now suppose that Alice searches slightly less frequently (with probability $\frac{1}{2}-\varepsilon$) than Bob (who searches with probability $\frac{1}{2}+\varepsilon$). The greedy algorithm first chooses the ad to assign to slot 1. Since the ad is more likely to be shown to Bob, the algorithm chooses ad 2, with an expected utility of $\frac{1}{2}+\varepsilon$. Since Alice will only look at position 1, no ad assigned to slot 2 can increase the expected utility further.  On the other hand, the optimal solution is to assign ad 1 to slot 1, and ad 2 to slot 2, with an expected utility of 1.}

\vspace{\presubsec}
\subsection {An algorithm with optimal approximation ratio} \label {sec:offline}
\vspace{\postsubsec}

We now present an algorithm that achieves the optimal approximation
ratio of $1 - 1/e$, improving on the $\frac 1 2$ approximation
for the locally greedy algorithm.  Our algorithm associates with each
partition $\Partition_\k$ a \emph {color} $\Color_\k$ from a palette
$\Colors$ of $\NumColors$ colors, where 
we use the notation $\bracket{n}=\set{1,2,\dots,n}$.   
For any set $S \subseteq \groundset \times \Colors$ and vector
$\Colorvec = (\Color_1, \ldots, \Color_\K)$, define
\ifthenelse{\boolean{istechrpt}}{
\[
    \sample_{\Colorvec}(S) = \bigcup_{\k=1}^\K \set { x \in \Partition_\k: (x, \Color_\k) \in S }  \mbox { .}\vspace{-1mm}
\]}{$\sample_{\Colorvec}(S) = \bigcup_{\k=1}^\K \set { x \in \Partition_\k: (x, \Color_\k) \in S }$.}

Given a set $S$ of (item, color) pairs, which we may think of as
labeling each item with one or more colors, $\sample_{\Colorvec}(S)$
returns a set containing each item $x$ that is labeled with whatever color
$\Colorvec$ assigns to the partition that contains $x$.
Let $F(S)$ denote the expected value of $f(\sample_{\Colorvec}(S))$ when each color $\Color_\k$ is selected uniformly at random from $\Colors$.  Our \OfflineGreedy algorithm greedily optimizes $F$, as shown in the following pseudocode.

\SetKw {KwForEach} {for each}
\SetKw {KwFrom} {from}
\SetKw {KwSet} {set}
\SetKw {KwReturn} {return}
\begin {algorithm}
\Titleofalgo { \OfflineGreedy }
\KwIn {integer $\NumColors$, sets $\Partition_1$, $\Partition_2$, \ldots, $\Partition_\K$, function $f: 2^\groundset \to \NonNegativeReals$ (where $\groundset = \bigcup_{\k=1}^\K \Partition_\k$)}
\ \\
\KwSet $G := \emptyset$. \\
\For(\tcc*[f]{For each color \hspace{1.2em} }){$\c$ \KwFrom $1$ \KwTo $\NumColors$}{
  \For(\tcc*[f]{For each partition}){$\k$ \KwFrom $1$ \KwTo $\K$} {
    \KwSet $g_{\k,\c} = \argmax_{x \in \P_\k  \times \set{\c}} \set { F(G
    + x) }$ \tcc*[f]{Greedily pick $g_{\k,\c}$\hspace{0.4em} } \\
    \KwSet $G := G + g_{\k,\c}$\;
  }
}
\KwForEach $\k \in \bracket{\K}$, choose $\Color_\k$ uniformly at random
    from $\Colors$. \\
\KwReturn $\sample_{\Colorvec}(G)$, where  $\Colorvec := (\Color_1, \ldots, \Color_\K)$.
\end {algorithm}
Observe that when $\NumColors = 1$, there is only one possible choice for $\Colorvec$, 
and \OfflineGreedy is simply the locally greedy
algorithm from \secref{sec:locally_greedy}.
\andreas{Perhaps move this paragraph to the Continuous Greedy section to clarify the relationship between the algorithms?}
In the limit as $\NumColors \to \infty$, \OfflineGreedy can intuitively be viewed
as an algorithm for a continuous extension of the
problem followed by a rounding procedure, in the
same spirit as Vondr\'{a}k's \ContinuousGreedy
algorithm~\citep{calinescu07}.
In our case, the continuous extension is to compute a probability
distribution $D_\k$ for each position $\k$ with support in $P_\k$ (plus
a special ``select nothing'' outcome), such that if we independently
sample an element $x_\k$ from $D_\k$, 
$\E {f(\set{x_1, \ldots, x_\K})}$ is maximized.
It turns out that if the positions individually, greedily, and in
round-robin fashion, add infinitesimal units of probability mass to their
distributions so as to maximize this objective function, they achieve the
same objective function value as if, rather than making decisions in a
round-robin fashion, they had cooperated and added the combination of $\K$
infinitesimal probability mass units (one per position) that greedily maximizes the
objective function.  The latter process, in turn, can be shown to be
equivalent to a greedy algorithm for maximizing a (different) submodular
function subject to a cardinality constraint, which implies that it
achieves a $1-1/e$ approximation ratio~\citep{nemhauser78}.
\ignore{  %
It turns out that if the positions individually, simultaneously, and greedily add probability mass 
at a fixed rate to their distributions to maximize this objective function, they do
nearly as well as a process that adds probability mass
simultaneously at the same rate to all positions in a manner which is greedy with
respect the collective increase in the objective function.
Moreover, the latter process can be viewed as the 
greedy algorithm to maximize a (different) submodular function subject to a 
cardinality constraint~\footnote{More precisely, as the limit process
of the greedy algorithm as $\delta \to 0$, where in each step the
algorithm selects an action that adds $\delta$ probability mass to each $D_\k$, and can perform at most
$1/\delta$ such actions.}, and is thus a $1-1/e$
approximation~\citep{nemhauser78}.
} %
\OfflineGreedy represents a tradeoff between these two extremes;
its performance is summarized by Theorem~\ref{thm:offline_greedy}. 
For now, we assume that the $\argmax$ in the inner loop is computed exactly.
\ifthenelse{\boolean{istechrpt}}
{In Appendix A}
{In the \fullversion~\citep{arxiv-version},}
we bound the performance loss that results from approximating the $\argmax$
(e.g., by estimating $F$ by repeated sampling).

\begin {theorem} \label {thm:offline_greedy}
Suppose $f$ is monotone submodular.  Then
$$F(G) \ge \beta(\K, \NumColors) \cdot \max_{\assignment \in \Feasible} f(\assignment),$$
where $\beta(\K, \NumColors)$ is defined as $1 - (1 - \frac {1} {\NumColors})^{\NumColors} - {\K \choose 2} \NumColors^{-1} $.
\end {theorem}

It follows that, for any $\varepsilon>0$, \OfflineGreedy achieves a $(1-1/e-\varepsilon)$ approximation factor using a number of colors that is polynomial in $\NumPartitions$ and $1/\varepsilon$.
The theorem will follow immediately from the combination of two key
lemmas, which we now prove.  Informally, Lemma
\ref {lem:outer_loop} analyzes the approximation error due to the
outer greedy loop of the algorithm, while Lemma \ref {lem:inner_loop}
analyzes the approximation error due to the inner loop.

\begin {lemma} \label {lem:outer_loop}
Let $G_{\c} = \set { g_{1,{\c}}, g_{2, {\c}}, \ldots, g_{\K, {\c}}  }$, and let $G^-_{\c} = G_1 \cup G_2 \cup \ldots \cup G_{{\c}-1}$.
For each color ${\c}$, choose $E_{\c} \in \Real$ such that
$
  F( G^-_{\c} \cup G_{\c}  )
  \ge \max_{x \in \Rows_{\c}} \set { F(G^-_{\c} + x) } - E_{\c}
$
where $\Rows_{\c} := \set{\Row : \forall \k \in \bracket{\K}, |\Row \cap (\P_\k \times \set {{\c}})| = 1}$ is the set of all possible
choices for $G_{\c}$.
Then
\begin {equation} \label {eq:guarantee}
    F(G) \ge \beta(\J) \cdot \max_{\assignment \in \Feasible} \set {f(\assignment)} - \sum_{\j=1}^\J E_\j \mbox { .}
\end {equation}
where $\beta(\J) = 1 - \paren {1 - \frac {1} {\J}}^\J$.
\end {lemma}
\andreas{Need more intuition and explanation on this lemma.}
\daniel{Punting on this.}
\begin {proof} (Sketch)
We will refer to an element $\Row$ of $\Rows_{\c}$ as a \emph {row}, and to ${\c}$
as the color of the row.  Let $\AllRows := \bigcup_{{\c}=1}^\J \Rows_{\c}$
be the set of all rows.  Consider the function $H:
2^{\AllRows} \to \NonNegativeReals$,
defined as $H(\Rows) = F \paren{ \bigcup_{\Row \in \Rows} \Row }$.  
We will prove the lemma in three steps:
(\emph{i}) $H$ is monotone submodular, 
(\emph{ii}) \OfflineGreedy is simply the locally greedy algorithm for finding a
    set of $\J$ rows that maximizes $H$,
where the $\j^{\text{th}}$ greedy step is performed with additive error $E_{\j}$,
    and 
(\emph{iii}) \OfflineGreedy obtains the guarantee \eqref
    {eq:guarantee} for maximizing $H$, and this implies the same ratio
    for maximizing $F$.

To show that $H$ is monotone submodular, it suffices to show that $F$ is monotone submodular.  Because $F(\assignment) = \Esub{\Colorvec}{f(\sample_{\Colorvec}(\assignment))}$, and because a convex combination of monotone submodular functions is monotone submodular, it suffices to show that for any particular coloring $\Colorvec$, the function $f(\sample_{\Colorvec}(\assignment))$ is monotone submodular.  This follows from the definition of $\sample$ and the fact that $f$ is monotone submodular.

The second claim is true by inspection.
To prove the third claim, we note that the row colors for a set of
rows $\Rows$ can be interchanged with no effect on $H(\Rows)$.
For problems with this special
property, it is known that the locally greedy algorithm obtains an
approximation ratio of $\beta(\J) = 1 - (1 - \frac 1 \J)^\J$~\citep{nemhauser78}.
Theorem 6 of \cite{streeter07tr} extends this result to handle
additive error, and yields  
\vspace{-2mm}
\[
    F(G) = H( \set {G_1, G_2, \ldots, G_\J} ) \ge \beta(\J) \cdot
\max_{\Rows \subseteq \AllRows: |\Rows| \le \J} \set { H(\Rows) } - \sum_{{\c}=1}^\J E_{\c} \mbox { .}
\vspace{-2mm}\]
To complete the proof, it suffices to show that
$\max_{\Rows \subseteq \AllRows: |\Rows| \le \J} \set { H(\Rows) }
\ge \max_{\assignment \in \Feasible} \set { f(\assignment) } $.  This follows from the
fact that for any assignment $\assignment \in \Feasible$, we can find a set $\Rows(\assignment)$ of $\J$
rows such that $\sample_{\cvec}(\bigcup_{\Row \in \Rows(\assignment)} R) = \assignment$ with
probability 1, and therefore $H(\Rows(\assignment)) = f(\assignment)$.
\end {proof}

 We now bound the performance of the the inner loop of \OfflineGreedy.

\begin {lemma} \label {lem:inner_loop}
Let $f^* = \max_{\assignment \in \Feasible} \set {f(\assignment)}$,
and let $G_{\c}$, $G^-_{\c}$, and $\Rows_{\c}$ be defined as in the statement of Lemma \ref
{lem:outer_loop}.  Then, for any ${\c} \in \bracket{\J}$,
\ifthenelse{\boolean{istechrpt}}{
\[
    F(G^-_{\c} \cup G_{\c}) \ge \max_{\Row \in \Rows_{\c}} \set {
    F(G^-_{\c} \cup \Row) } -  {\K \choose 2} \J^{-2} f^*
    \mbox { .}
\]}
{ $F(G^-_{\c} \cup G_{\c}) \ge \max_{\Row \in \Rows_{\c}} \set {
    F(G^-_{\c} \cup \Row) } -  {\K \choose 2} \J^{-2} f^*$. }
\end {lemma}
\andreas{Need more intuition and explanation on this lemma.}
\daniel{Punting on this.}
\begin {proof} (Sketch)
Let $N$ denote the number of partitions whose color (assigned by $\cvec$) is ${\c}$.  
For $\Row \in \Rows_{\c}$, let
$\Delta_{\cvec}(\Row) := f(\sample_{\cvec}(G^-_{\c} \cup \Row)) -
f(\sample_{\cvec}(G^-_{\c}))$, and let $F_{\c}(\Row) := F(G^-_{\c} \cup \Row) - F(G^-_{\c})$.  By definition,
    $F_{\c}(\Row) = \Esub{\cvec}{\Delta_{\cvec}(\Row)} = \Pr{N=1} \Esub{\cvec} {  \Delta_{\cvec}(\Row)  | N = 1 } + \Pr{N \ge 2} \Esub{\cvec}{\Delta_{\cvec}(\Row) | N \ge 2}$,
where we have used the fact that $\Delta_{\cvec}(\Row) = 0$ when $N =
0$.
The idea of the proof is that the first of these terms dominates as
$\J \to \infty$, and that $ \Esub{\cvec} {  \Delta_{\cvec}(\Row)  | N = 1 }$
can be optimized exactly
simply by optimizing each element of
$P_\k \times \set{{\c}}$ independently.
Specifically, it can be seen that $\Esub{\cvec}{  \Delta_{\cvec}(\Row)  | N = 1 } =
\sum_{\k=1}^\K f_\k(\Row \cap (P_\k \times \set{{\c}}))$ for suitable $f_\k$.
Additionally, $f_0(\Row) = \Pr{N \ge 2} \Esub{\cvec}{\Delta_{\cvec}(\Row) | N \ge
2}$ is a monotone submodular function of a set of (item, color) pairs, for the same reasons $F$ is.
Applying Lemma \ref {lem:locally_greedy} with these 
$\set{f_\k : \k \ge 0}$ yields
\ifthenelse{\boolean{istechrpt}}{
\[
    F_{\c}(G_{\c}) + \Pr{N \ge 2} \Esub{\cvec}{\Delta_{\cvec}(G_{\c}) | N \ge 2} \ge \max_{R \in \Rows_{\c}} \set { F_{\c}(R) }
    \mbox { .}
\]}
{$ F_{\c}(G_{\c}) + \Pr{N \ge 2} \Esub{\cvec}{\Delta_{\cvec}(G_{\c}) | N \ge 2} \ge \max_{R \in \Rows_{\c}} \set { F_{\c}(R) }$.}
To complete the proof, it suffices to show $\Pr{N \ge 2} \le {K
\choose 2} \J^{-2}$ and $\Esub{\cvec}{\Delta_{\cvec}(G_{\c}) | N \ge 2}
\le f^*$.
The first inequality holds because, if we let $M$ be the number of \emph{pairs} of partitions that are both assigned color ${\c}$, we have $\Pr{N \ge 2} = \Pr{M \ge 1}\le \E{M} = {\K \choose 2} \J^{-2}$.
The second inequality follows from the fact that for any
$\cvec$ we have $\Delta_{\cvec}(G_{\c}) \le f(\sample_{\cvec}(G^-_{\c} \cup G_{\c})) \le
f^*$.
\end {proof}

\section{An algorithm for online learning of assignments} \label {sec:online}

We now transform the offline algorithm of \secref {sec:offline} into
an online algorithm.  The high-level idea behind this transformation
is to replace each greedy decision made by the offline algorithm with
a no-regret online algorithm.  A similar approach was used by
\cite{radlinski08} and \cite{streeter08} to obtain an online
algorithm for different (simpler) online problems.  \andreas{We should probably explain in more detail how exactly we plug in existing no-regret algorithms, and which properties we require, perhaps briefly reviewing WMR.}
\SetKw {KwEach} {each}
\begin {algorithm}
\Titleofalgo { \OnlineGreedy \ \ (described in the full-information feedback model) }
\KwIn {integer $\C$, sets $\P_1$, $\P_2$, \ldots, $\P_\K$}
\ \\
\KwForEach $\k \in \bracket{\K}$ and $\c \in \Colors$, let $\Alg_{\k,\c}$ be
a no-regret algorithm with action set $\P_\k \times \set{\c}$. \\
\For {$t$ \KwFrom $1$ \KwTo $T$} {
  \KwForEach $\k \in \bracket{\K}$ and $\c \in \bracket {\J}$, let $g^t_{\k,\c} \in \P_\k \times \set{\c}$ be the action selected by $\Alg_{\k,\c}$ \\
    \KwForEach $\k \in \bracket{\K}$, choose $\Color_\k$ uniformly at
    random from $\bracket {\J}$.  Define $\Colorvec = (\Color_1, \ldots,
    \Color_\K)$. \\
    select the set $G_t = \sample_{\Colorvec} \paren { \set { g^t_{\k,\c}: \k \in \bracket{\K}, \c \in \bracket{\J}  } }$ \\
    observe $f_t$, and let $\bar F_t(\assignment) := f_t(\sample_{\Colorvec}(\assignment))$ \\
    \For {\KwEach $\k \in \bracket{\K}$, $\c \in \bracket{\J}$} {
    define $G^{t-}_{\k,\c} \equiv \set {g^t_{\k',\c'}:
    \k' \in \bracket{\K}, \c' < \c} \cup \set {g^t_{\k',\c}: \k' < \k }$ \\
    \KwForEach $x \in \P_\k \times \set{\c}$, feed back $\bar F_t(G^{t-}_{\k,\c} + x)$ to $\Alg_{\k,\c}$ as the reward for choosing $x$
  }
}
\end {algorithm}

The following theorem summarizes the performance of \OnlineGreedy.

\begin {theorem} \label {thm:online_greedy}
Let $\Regret_{\k,\c}$ be the regret of $\Alg_{\k,\c}$, and let $\beta(\K,
\J) = 1 - \paren { 1 - \frac 1 \J }^\J - {\K \choose 2} \J^{-1}$.
Then
\[
    \E { \sum_{t=1}^T f_t(G_t) } \ge \beta(\K, \J) \cdot \max_{\assignment \in \Feasible} \set { \sum_{t=1}^T f_t(\assignment) } - \E { \sum_{\k=1}^\K \sum_{\c=1}^\J \Regret_{\k,\c} } \mbox { .}
\]
\end {theorem}

Observe that Theorem \ref {thm:online_greedy} is similar to Theorem \ref {thm:offline_greedy}, with the addition of the $\E{r_{\k,\c}}$ terms.  The idea of the proof is to view \OnlineGreedy as a version of
\OfflineGreedy that, instead of greedily selecting single (element,color) pairs
$g_{\k,\c} \in \P_\k \times \set {\c}$, greedily selects (element vector,
color) pairs $\vec g_{\k,\c} \in \P_\k^T \times \set {\c}$ 
(here, $\P_\k^T$ is the $T^{\text{th}}$ power of the set $P_\k$).  We
allow for the case that the greedy decision is made imperfectly, 
with additive error $r_{\k,\c}$; this is the source of the extra terms.  Once this
correspondence is established, the theorem follows along the lines of Theorem~\ref {thm:offline_greedy}.
\ifthenelse{\boolean{istechrpt}}
{For a proof, see Appendix A.}
{For a proof, see the \fullversion~\citep{arxiv-version}.}

\begin {cor} \label {cor:WMR}
If \OnlineGreedy is run with randomized weighted majority \citep {cesabianchi97}
as the subroutine, then
\[
    \E { \sum_{t=1}^T f_t(G_t) } \ge \beta(\K, \J) \cdot \max_{\assignment \in \Feasible} \set { \sum_{t=1}^T f_t(\assignment) } - \bigO { \J \sum_{\k=1}^\K \sqrt { T \log \size { \Partition_\k } } }  \mbox { .}
\]
where $\beta(\K, \J) = 1 - \paren { 1 - \frac 1 \J }^\J - {\K \choose
2} \J^{-1}$.
\end {cor}

Optimizing for $\J$ in Corollary~\ref{cor:WMR}
yields $(1-\frac{1}{e})$-regret
$\tilde\Theta( \K^{3/2} T^{1/4} \sqrt{\OPT} )$ ignoring logarithmic
factors, where $\OPT := \max_{\assignment \in \Feasible} \set
{ \sum_{t=1}^T f_t(\assignment) }$ is the value of the static optimum.

\vspace{\prepara}
\paragraph {Dealing with bandit feedback.}

\label {sec:limited_feedback}

\OnlineGreedy can be modified to work in the bandit feedback model. 
The idea behind this modification is that on each round we ``explore"
with some small probability,
in such a way that on each round we obtain an unbiased estimate of the
desired feedback values $\bar F_t(G^{t-}_{\k,\c} + x)$ for each $\k \in
\bracket{\K}$, $\c \in \Colors$, and $x \in \Partition_\k$.  This
technique can be used to achieve a bound similar to the one stated in Corollary \ref {cor:WMR}, but with an additive regret term of
$\bigO{\paren{T \size{\groundset} \J \K}^\frac 2 3 \paren{\log \size{\groundset}}^\frac 1 3}$.

\vspace{\prepara}
\paragraph {Stronger notions of regret.}

By substituting in different algorithms for the subroutines
$\Alg_{\k,\c}$, we can obtain additional guarantees.  For example, 
\citet {blum07} consider online problems in which we are
given \emph{time-selection functions} $I_1, I_2, \ldots, I_M$.  Each
time-selection function $I: \bracket {T} \rightarrow [0,1]$ associates
a weight with each round, and defines a corresponding weighted notion
of regret in the natural way.  Blum and Mansour's algorithm
guarantees low weighted regret with respect to all $M$ time selection
functions simultaneously.  This can be used to obtain low regret with
respect to different (possibly overlapping) windows of time
simultaneously, or to obtain low regret with respect to subsets of
rounds that have particular features.  By using their algorithm as a
subroutine within \OnlineGreedy, we get similar guarantees,
both in the full information and bandit feedback models.

\section{Handling arbitary matroid constraints}\label {sec:continuous-greedy}
\andreas{This section is long, and not the main focus of the paper. I
  would recommend moving \thmref{thm:offline_cont_greedy} and
  \lemref{lem:cont_greedy_error} to the appendix.}
\daniel{Punting on this.}
Our offline problem is known as \emph{maximizing a monotone submodular
function subject to a (simple) partition matroid constraint} in the
operations research and theoretical computer science communities.
The study of this problem culminated in the elegant
$(1-1/e)$ approximation algorithm of~\cite{vondrak08} and
a matching unconditional lower bound of~\citet{mirrokni08}.  \vondrak's algorithm, called
the \ContinuousGreedy algorithm,
has also been extended to
handle arbitrary matroid constraints~\citep{calinescu07}.  
The \ContinuousGreedy algorithm,
however, cannot be applied to our problem directly,
because it requires the ability to sample $f(\cdot)$ on \emph{infeasible}
sets $S \notin \Feasible$.  In our context, this means it must have
the ability to ask (for example) what the revenue will be if ads $a_1$
and $a_2$ are placed in position $\#1$  \emph{simultaneously}.
We do not know how to answer such questions in a way that leads to
meaningful performance guarantees, and hence we developed the \OfflineGreedy
algorithm to circumvent this difficulty.  However, in some
applications, it may be possible to evaluate $f(\cdot)$ on infeasible
sets $S \notin \Feasible$ at the end of each round, 
particularly in the full information setting.
As we will now show, in this case we may implement an online version
of the \ContinuousGreedy algorithm, which has the virtue of
being able to deal with arbitrary matroid constraints.

\subsection{Background: Matroid Constraints and the Continuous Greedy
  Algorithm} \label{sec:background-cont-greedy}

A \emph{matroid} $\matroid = (\matroidgroundset, \indset)$ consists of
a finite ground set $\matroidgroundset$ and a nonempty collection of \emph{independent
  sets} $\indset \subseteq 2^{\matroidgroundset}$ such that 
($i$) $A \subset B$ and $B \in \indset$ implies $A \in \indset$, and
($ii$) $A, B \in \indset$ and $|A| < |B|$ implies there exists some $b
\in B \setminus A$ such that $A + b \in \indset$.  
Matroids play an
important role in the theory of optimization, where they generalize
the notion of linear independence in vector spaces; see e.g.~\cite{SchrijverB}.

We are interested in problems of the form 
\begin{equation}
  \label{eqn:matroid-constraint}
  \max \set{f(\assignment) : \assignment \in \indset}
\end{equation}
for a monotone submodular function $f:2^{\matroidgroundset} \to
\NonNegativeReals$ such that $f(\emptyset) = 0$ 
and a matroid $(\matroidgroundset, \indset)$.
This is known as maximizing $f$ subject to a matroid constraint.
Note our offline assignment problem is a special case with a so-called
\emph{simple partition matroid constraint}, so that 
$\matroidgroundset = \Partition_1\cup\Partition_2 \cup
\dots\cup\Partition_\NumPartitions$, where $\Partition_k$ is
the set of items suitable for position $k$, and 
$\indset = \set{\assignment : \forall k,\ |\assignment \cap \Partition_\k| \le 1  }$
is the set of feasible assignments.

To obtain a $(1-1/e)$ approximation to
\probref{eqn:matroid-constraint}, \cite{calinescu07}~use \vondrak's
\ContinuousGreedy algorithm. 
This algorithm defines a
continuous relaxation of the problem, obtains a $(1-1/e)$
approximation to the relaxation, and then rounds the resulting
fractional solution to a feasible solution to the original problem
without losing any objective value in expectation.  We discuss each
step in turn.

The continuous relaxation of~\probref{eqn:matroid-constraint} that the
\ContinuousGreedy algorithm works with has a feasible set consisting of 
the \emph{matroid polytope} of the constraint matroid $\matroid = (\matroidgroundset, \indset)$, denoted $\polytope{\matroid}$.  It is
defined as the convex hull of all characteristic vectors of
independent sets in $\matroid$, i.e., 
$$
\polytope{\matroid} := \conv \set{\charvec{S} : S \in \indset}
$$
and lies in $\reals^{\matroidgroundset}$.  The objective is the smooth
\emph{multilinear extension} of the monotone submodular objective
function $f:2^{\matroidgroundset} \to \NonNegativeReals$, which is a
function 
$F:[0,1]^{\matroidgroundset} \to \NonNegativeReals$ defined as
follows.
For $y \in [0,1]^{\matroidgroundset}$, let $\assignment_{y} \subseteq
\matroidgroundset$ be a random set such that each $v \in
\matroidgroundset$ is included in $\assignment_{y}$ independently with
probability $y_{v}$.  Then the multilinear extension of $f$ is defined 
as 
\begin{equation} \label{eqn:multilinear-ext}
F(y) := \expct{f(\assignment_{y})} = \sum_{\assignment \subseteq
  \matroidgroundset} f(\assignment) \prod_{v \in \assignment} y_v
\prod_{v \notin \assignment} \paren{1 - y_v} 
\end{equation}
\cite{calinescu07} show that for all $u,v \in \matroidgroundset$ the
multilinear extension satisfies
$\dfrac{\partial F}{\partial y_u} \ge 0$ and $\dfrac{\partial^2 F}{\partial y_u \partial y_v} \le 0$.
Given the definitions of $\polytope{\matroid}$ and $F$, we can define
the continuous relaxation of~\probref{eqn:matroid-constraint} as
\begin{equation}
  \label{eqn:continuous-relaxation}
  \max \set{F(y) : y \in \polytope{\matroid}}
\end{equation}

This relaxation is NP-hard, and indeed is as hard to approximate as
the original problem.  However, \cite{vondrak08} shows how a $(1-1/e)$
approximation may be obtained as follows.
Define a parameterized curve $\set{y(t) : t \in [0,1]} \subset
[0,1]^{\matroidgroundset}$ satisfying $y(0) = \mathbf{0}$ and 
$\dfrac{dy}{dt} \in \argmax_{z \in \polytope{\matroid}} \paren{ z \cdot
  \nabla F(y)}$.  Then $y(1) \in \polytope{\matroid}$ since
$y(1) = \int_{t=0}^1 \paren{\dfrac{dy}{dt}} dt$ is a convex
combination of vectors in $\polytope{\matroid}$.
Moreover, 
$F(y(1)) \ge (1 - 1/e)\, \OPT$, where 
$\OPT = \max \set{F(y') : y' \in \polytope{\matroid}}$ is the optimal
value of~\probref{eqn:continuous-relaxation}.
This is proved by demonstrating that any $y \in  \polytope{\matroid}$
there is a direction $v \in \polytope{\matroid}$ such that 
$v \cdot  \nabla F(y) \ge \OPT - F(y)$.
Given this fact, it is easy to show that $F(y(t))$ dominates the
solution to the differential equation $d\phi/dt \ge \OPT - \phi$ with
boundary condition $\phi(0) = 0$, whose solution is 
$\phi(t) = (1 - e^{-t}) \OPT$.  Hence $F(y(1)) \ge (1 - 1/e)\, \OPT$.
To implement this procedure, 
first note that given $\nabla F(y)$,
computing $\argmax_{z \in \polytope{\matroid}} \paren{ z \cdot
  \nabla F(y)}$ is easy, as it amounts to optimizing a linear function
over $\polytope{\matroid}$; in fact an integral optimum may be found
using a simple greedy algorithm.
However, two complications arise.
The first is that this continuous process must be discretized.
For technical reasons, \citet{calinescu07} replace all occurrences of the gradient $\nabla F(y)$
in the continuous process with a related
quantity called the \emph{marginal}, $\Delta F(y) \in \NonNegativeReals^{\matroidgroundset}$,
defined coordinate-wise by 
$$(\Delta F(y))_v = \expct{f(S_{y} \cup \set{v}) - f(S_y)} = (1 -
y_{v})(\nabla F(y))_v.$$
\newcommand{\ya}{\ensuremath{y^{(1)}}}
\newcommand{\yb}{\ensuremath{y^{(2)}}}
\daniel{Here's a bad example for the gradient:  $\ya := y(\tau - \delta) =
  (1-\delta) \charvec{\matroidgroundset}$, 
$\yb := \charvec{\matroidgroundset}$, 
$\indset = 2^{\matroidgroundset}$, and $f(S) = \min \set{1, |S|}$.
Then $F(\yb) - F(\ya) = \delta^n$, with $n = |\matroidgroundset|$, yet 
$\delta \charvec{\matroidgroundset} \cdot \nabla F(\ya) = n
\delta^{n}$.
So we cannot get marginal benefit $\delta (1 - d \delta)
\charvec{\matroidgroundset} \cdot \nabla F(\ya)$, or even close to it.
}
The second complication arises because 
$\Delta F(y)$ cannot be computed
exactly given oracle access to $f$, but must be estimated via random sampling.
\citet{calinescu07} choose to take 
enough samples are taken so that by Chernoff bounds it is
likely that the estimate of each coordinate of $\Delta F(y(t))$ for
each iteration $t \in \set{0, \delta, 2\delta, \ldots, 1}$ is accurate
to up high precision, namely $\delta\, \OPT$.

\begin {algorithm}\label{alg:cont-greedy}
\SetKwFunction{PipageRound}{PipageRound}
\Titleofalgo { Continuous Greedy (\cite{vondrak08,calinescu07}) }
\KwIn {matroid $\matroid = (\matroidgroundset,
  \indset)$, monotone submodular $f: 2^\matroidgroundset \to
  \NonNegativeReals$, $\delta \in \set{1/n : n \in \nats}$, $\nsamp \in \nats$}
\ \\
\KwSet $y(0) = \mathbf{0}$. \\
\For{$t$ \KwFrom $\delta$ \KwTo $1$ \emph{\textbf{in increments
of}} $\delta$}{
  \For{\KwEach $v \in \matroidgroundset$} {
    Compute estimates $\est_v(t-\delta)$ of 
  $(\Delta F(y(t-\delta)))_v = F(y(t-\delta)
        + (1-y(t-\delta)_v) \charvec{\set{v}}) - F(y(t-\delta))$ 
via random
      sampling, by taking the average of $f(\assignment_{y(t-\delta)}
      \cup \set{v}) - f(\assignment_{y(t-\delta)})$
      over $\nsamp$ independent samples of
      $\assignment_{y(t-\delta)}$ as described in 
(\ref{eqn:multilinear-ext})\;
   }
   Let $I(t)$ be a maximum-weight independent set in $\matroid$,
   according to the weights $\est_v(t-\delta)$.  This may be computed via a simple
   greedy algorithm.\\  
   \KwSet $y(t) = y(t-\delta) + \delta\, \charvec{I(t)}$.
}
\KwReturn $\PipageRound(y(1))$\;
\caption{Pseudocode for the \ContinuousGreedy algorithm.  
For default parameters, \cite{calinescu07} select $\delta =
1/9d^2$ where $d = \max \set{|S| :S \in \indset}$ is the rank of
$\matroid$, and $\nsamp = \frac{10}{\delta^2}(1 + \ln
|\matroidgroundset|)$ samples per marginal coordinate. }
\end {algorithm}

Finally, given a fractional solution $y$, the 
\ContinuousGreedy algorithm rounds it to an integral solution which is the
characteristic vector of a feasible set, using a rounding technique
called \emph{pipage rounding}.   Refer to~\citep{calinescu07} for details.  

\subsection{The Online Continuous Greedy Algorithm} \label{sec:OCG}

In this section we develop an online version of the 
\ContinuousGreedy algorithm.  We begin with an observation that 
the (randomized variant of the) pipage rounding scheme
in~\cite{calinescu07} is \emph{oblivious}, in the sense
that it does not require access to the objective function.
This was emphasized by \cite{calinescu07}, who point out its advantage in 
settings where we have only approximate-oracle access to $f$.  For the
same reason, it is extremely convenient for the design of the online 
continuous greedy algorithm. 
In particular, once we obtain a fractional solution $y$ to
\probref{eqn:continuous-relaxation}, we can round it to a random feasible
solution $\assignment \in \indset$ for the original problem such that 
$\expct{f(\assignment)} \ge F(y)$ \emph{before} gaining access to $f$.
We elect to do so, and hence the final rounding step is identical for
the \ContinuousGreedy algorithm and our online version of it.

Hence we may focus on developing an online algorithm for
the online version of the continuous relaxation, \probref{eqn:continuous-relaxation}.
Hence there is a stream of nonnegative monotone submodular functions
$\set{f_t}_{t=1}^{T}$ with multilinear extensions
$\set{F_t}_{t=1}^{T}$, and a fixed matroid $\matroid =
(\matroidgroundset, \indset)$, and in each time step $t$ we must
output some $y^t \in \polytope{\matroid}$ with the goal of 
maximizing $\sum_{t=1}^T F_t(y^t)$, and after outputting $y^t$ we then
receive oracle access to $f_t$.  We assume here the full information
model, in which we have oracle access to (i.e., we may efficiently compute)
$f_t(S)$ for arbitrary $S \subseteq \matroidgroundset$.

Consider the \ContinuousGreedy algorithm.  There are
$1/\delta$ stages of the algorithm, and in each stage $\tau$
it seeks a set $I(\tau) \in \indset$ 
maximizing a linear objective, namely $\charvec{I(\tau)} \cdot \Delta
F(y(\tau - \delta))$.  Our online variant will thus need to solve, for
each stage $\tau$, the problem of the online maximization of linear
functions subject to a matroid constraint.
For this, the \emph{follow the perturbed leader}
algorithm~\citep{KV05} is suitable. 
On each round $t$, this algorithm simply outputs the feasible solution
maximizing $g_0 + \sum_{1 \le t' < t} g_{t'}$, where $g_{t'}$ is the linear
objective seen in round $t'$, and $g_0$ is a random
linear objective drawn from a suitably chosen distribution.
To ensure that these algorithms have no-regret, we must offer suitable
feedback to them at the end of each round.  The objective faced by
the algorithm for stage $\tau$ will be 
$S \mapsto \charvec{S} \cdot \Delta F(y^t(\tau - \delta))$, where 
$y^t(\tau - \delta)$ is the point reached in the
algorithm's trajectory through $\polytope{\matroid}$ for that round at
stage $\tau$, i.e., $\delta$ times the sum of characteristic vectors
of the outputs of the earlier stages.
Unfortunately, we cannot compute $\Delta F(y^t(\tau - \delta))$
exactly.  We will defer the details of how to address this point
to~\secref{sec:feedback}.
For now, suppose we can in fact feed back $S \mapsto \charvec{S} \cdot \Delta F(y^t(\tau - \delta))$.

\begin {algorithm}\label{alg:online-cont-greedy}
\SetKwFunction{PipageRound}{PipageRound}
\Titleofalgo { Online Continuous Greedy}
\KwIn {matroid $\matroid = (\matroidgroundset,
  \indset)$, sequence of monotone submodular functions $f_t: 2^\matroidgroundset \to
  \NonNegativeReals$, $\delta \in \set{1/n : n \in \nats}$} %
\ \\
Initialize $1/\delta$  Perturbed Follow the Leader algorithms 
$\set{\Alg_{\tau} : \tau \in \set{\delta, 2\delta, \ldots, 1 }}$ for maximizing a linear objective over $\indset$.\\
\For{$t$ \KwFrom $1$ \KwTo $T$}{
\For{\KwEach $\tau \in \set{\delta, 2\delta, \ldots, 1}$}{
        Let $I(\tau) \in \indset$ be the set selected by $\Alg_{\tau}$.}
  Let $y = \sum_{s=1}^{1/\delta} \delta\, \charvec{I(s\delta)}$.\\
  Select $S_t = \PipageRound(y)$\;
  Obtain reward $f_t(S_t)$ and oracle access to $f_t$.\\
  \For{\KwEach $\tau \in \set{\delta, 2\delta, \ldots, 1}$}{
   If $\tau = \delta$ let $y(\tau) = \mathbf{0}$, otherwise let $y(\tau) = \sum_{s=1}^{\tau/\delta -1} \delta \charvec{I(s
    \delta)}$. \\ 
       Generate an unbiased estimate $\est_{\tau}(t)$ for $\Delta
       F_t(y(\tau))$, where $F_t$ is the multilinear extension of
       $f_t$.
       Feed back the linear objective function $S \mapsto (\charvec{S} \cdot
       \est_{\tau}(t))$ to $\Alg_{\tau}$.
   }
}
\caption{The \OnlineContinuousGreedy algorithm.}
\end {algorithm}

\subsubsection{Analysis of the Continuous Greedy Algorithm
  with Noise} \label{sec:cg-noise}

\cite{streeter08} introduced an analysis framework based on
``meta-actions'', in which an online process over $T$ rounds
is interpreted as a single run in a combined instance of a suitable
offline problem with objective $\sum_{t=1}^T F_t$, in which the
algorithm makes errors.  Using this framework, we will analyze the \ContinuousGreedy algorithm with noise in much the same way we analyzed \OfflineGreedy.

Recall that in the analysis of \OfflineGreedy (specifically, in the
proof of Lemma~\ref{lem:outer_loop}) we considered problem of 
selecting a set of $\NumColors$ \emph{rows} to maximize a monotone
submodular function $H$, where each row corresponds to a way of
filling in a row of the algorithm's table of items (i.e., a choice of
item for each partition).  
\OfflineGreedy is then interpreted as a noisy version of the locally greedy
algorithm for this problem.
We analyze \ContinuousGreedy likewise, where the role of a row is
played by a set in $\indset$, 
and the monotone submodular objective function $H:2^{\indset \times \delta\nats} \to
\NonNegativeReals$ defined by 
$H(\Rows) = F(\sum_{(S, \tau) \in \Rows} \delta \charvec{S})$, where $F$
is the multilinear extension of the original objective $f$.
Note that $\max_{\Rows : |\Rows| \le 1/\delta} \paren{H(\Rows)} =
\max_{S \in \indset} \paren{f(S)}$, since we can use pipage rounding
to round any $y = \sum_{(S, \tau) \in \Rows} \delta \charvec{S}$ to
some $S \in \indset$ with $f(S) \ge F(y)$, and given any $S \in
\indset$ we can construct 
$\Rows(S) = \set{(S, \tau) : \tau \in \delta \nats, 0 \le
\tau \le 1}$ which has the property that 
$H(\Rows(S)) := F(\sum_{(S, \tau) \in \Rows}
\delta \charvec{S}) = F(\charvec{S}) = f(S)$.

Now, let $I({\tau})$ be the set chosen in round $\tau$ of
\ContinuousGreedy, and define the stage $\tau$ error $\err_{\tau}$ to
be such that 
$$
H(\set{I(\tau') : \tau' \le \tau}) - H(\set{I(\tau') : \tau' < \tau})
\ge 
\delta\paren{\OPT - H(\set{I(\tau') : \tau' < \tau})} -
\err_{\tau}.$$
where $\OPT = \max_{\Rows : |\Rows| \le 1/\delta} \paren{H(\Rows)}$.
We obtain the following result.

\begin {theorem} \label {thm:offline_cont_greedy}
Suppose $f:2^{\matroidgroundset} \to \NonNegativeReals$ is monotone
submodular.  
Let $S$ be the output of the noisy \ContinuousGreedy
algorithm with errors $\set{\err_{\tau} : \tau \in \delta \nats, 0 \le
\tau \le 1}$ run on
matroid $(\matroidgroundset, \indset)$.
Then
$$\expct{f(S)} \ge \paren{1 - \frac{1}{e}} \max_{\assignment \in
  \indset} \set {f(\assignment)} - \sum_{\tau} \err_{\tau}.$$
\end {theorem}

\begin{proof}
\ContinuousGreedy can be viewed as a noisy version of the locally 
greedy algorithm for selecting the set of $1/\delta$ rows maximizing   
the monotone submodular objective function $H:2^{\indset \times \delta\nats} \to
\NonNegativeReals$ is defined by 
$H(\Rows) := F(\sum_{(S, \tau) \in \Rows} \delta \charvec{S})$, where
$F$ is the multilinear extension of the original objective $f$.
Note the order of the rows does not matter, since 
from the definition of $H$, it is clear that permuting the order of rows has
no effect on the objective, i.e., for all permutations $\pi$, we have 
$H( \set{(S_\tau, \tau) : \tau \in \delta \nats, 0 \le
\tau \le 1}  ) = H(
\set{(S_{\tau}, \pi(\tau)) : \tau \in \delta \nats, 0 \le
\tau \le 1}  ) $.
That $H$ is monotone submodular follows easily from the facts that 
$\dfrac{\partial F}{\partial y_u} \ge 0$ and $\dfrac{\partial^2
  F}{\partial y_u \partial y_v} \le 0$~\citep{calinescu07}.
Hence from Theorem~$6$ of~\cite{streeter07tr}, which bounds the
performance of a noisy version of the locally greedy algorithm for
precisely this problem, we obtain
\begin{equation}
  \label{eq:2dd}
  F(y(1)) = H(\set{I(\tau) : \tau \in \delta \nats, 0 \le
\tau \le 1} ) \ge \paren{1 - \frac{1}{e}} \max_{\Rows : |\Rows| \le 1/\delta} \set {H(\Rows)} - \sum_{\tau} \err_{\tau}.
\end{equation}
As we proved above when introducing $H$, $\max_{\Rows : |\Rows| \le 1/\delta} \set {H(\Rows)} =
\max_{S \in \indset} f(S)$.
Hence $F(y(1)) \ge \paren{1 - \frac{1}{e}} \max_{S \in \indset} f(S) -
\sum_{\tau} \err_{\tau}$, and to complete the proof it suffices to
note that pipage rounding results in an
output $S$ with $\expct{f(S)} \ge F(y(1))$. 
\end{proof}

We distinguish two different sources of error: that arising 
because of discretization, and that arising from not computing 
$\argmax_{S \in \indset} \set{\charvec{S} \cdot \Delta F(y(\tau -
  \delta)) }$ exactly for whatever reason.  
The \ContinuousGreedy algorithm, even when it is
given a perfect estimate of the vector $\Delta F$ and computes
$\argmax_{S \in \indset} \set{\charvec{S} \cdot \Delta F}$ exactly, 
will not
generally achieve zero error.  This is because maximizing 
$\charvec{S} \cdot \Delta F(y(\tau - \delta))$ is not equivalent to
maximizing $F(y(\tau)) - F(y(\tau - \delta))$, due to the nonlinearity
of $F$.  We bound this source of error as in
Lemma~\ref{lem:inner_loop}, using a slight variation of an argument
of~\cite{calinescu07}.

\begin {lemma} \label {lem:cont_greedy_error}
Let $\OPT = \max_{\assignment \in \Feasible} \set {f(\assignment)}$.
Fix $\tau$ and arbitrary $I(\delta), I(2\delta), \ldots, I(\tau -
\delta) \in \indset$.
Let $I(\tau) \in \indset$ and $\err'_{\tau}$ be such that 
$\charvec{I(\tau)} \cdot \Delta
F(y(\tau - \delta))
\ge  \max_{S \in \indset} \charvec{S} \cdot \Delta
F(y(\tau - \delta)) - \err'_{\tau}$, where 
$y(\tau - \delta) = \sum_{\tau' < \tau} \delta \charvec{I(\tau')}$ as
in the description of \ContinuousGreedy.
Let $I^{-}(\tau) := \set{ (I(\tau'), \tau') : \tau' \in \delta \nats,
  \tau' < \tau } $.
Then, 
\[
 H(I^{-}(\tau) \cup \set{(I(\tau), \tau)}) -  H(I^{-}(\tau)) \ge 
\delta \paren{\OPT -  H(I^{-}(\tau))} - d \delta^2 \OPT  - \delta \err'_{\tau}
\]
where $d = \max \set{|S| : S \in \indset}$ is the rank of $\matroid$.
\end {lemma}

\begin {proof} %
For any $S$, define $\Delta(S) := H(I^{-}(\tau) \cup \set{(S, \tau)}) -
H(I^{-}(\tau))$, where 
$H(\Rows) := F(\sum_{(S, \tau) \in \Rows} \delta \charvec{S})$.
Let $\ya = y(\tau - \delta) = \sum_{\tau' < \tau} \delta
\charvec{I(\tau')}$, $\yb = \ya + \delta
\charvec{I(\tau)}$.  Then $\Delta(I(\tau)) = F(\yb) - F(\ya)$.
Since $F$ is monotone, we may bound $\Delta(I(\tau))$ from below by 
$\Delta(I(\tau)) \ge F(z) - F(\ya)$ for any $z \le \yb$.
We select $z$ such that $z_v = \yb_v$ for all $v \notin I(\tau)$, and 
$z_v = y_v + \delta - \delta y_v = 1 - (1 - y_v)(1 - \delta)$ for all
$v \in I(\tau)$.
Recall $S_y$ is a random set such that each $v$ is in $S_y$ independently with
probability $y_{v}$, and $F(y) = \expct{f(S_y)}$.
We apply the following interpretation to $F(z)$:
Sample $A = S_{\ya}$ and then independently sample $B = S_{(\yb - \ya)}$.
Then the random set $A \cup B$ has the same distribution as $S_z$.
We then bound 
\begin{eqnarray}
  \label{eq:5ddd}
  \Delta(I(\tau)) & \ge & \expct{f(A \cup B) - f(A)} \\
           & \ge & \sum_{v \in I(\tau)} \expct{f(A \cup B) - f(A) \mid
             B = \set{v}} \cdot \Pr{  B = \set{v} } \\
           & = & \sum_{v \in I(\tau)} (\Delta F(\ya))_v \paren{ \delta (1
           - \delta)^{|I(\tau)| - 1}} \\
           & \ge & \delta (1 - d \delta) \charvec{I(\tau)} \cdot
           \Delta F(\ya) \\
           & \ge &  \delta (1 - d \delta)  \paren{\max_{S \in \indset} \paren{\charvec{S} \cdot \Delta
F(\ya)} - \err'_{\tau}} \label{eq:5:line4}
\end{eqnarray}
Here we have used that $|I(\tau)| \le d$, and that 
$(1 - \delta)^{d-1} \ge (1 - d \delta)$.  The latter follows easily from
Bernoulli's inequality.

Next we claim that 
\begin{equation}
  \label{eq:6ddd}
  \max_{ S \in \indset} \paren{\charvec{S}
\cdot \Delta F(y(\tau - \delta))} \ge \OPT - F(y(\tau - \delta)) =  \OPT - H(I^-(\tau))
\end{equation}
Fix an optimal solution $S^* \in \indset$, and let $A$ be in the
support of $S_{y(\tau - \delta)}$.
By the monotonicity and submodularity of $f$, we have 
$f(S^*) \le f(A \cup S^*) \le f(A) + \sum_{v \in S^*} \paren{f(A  + v)
  - f(A)}$.  
Taking the expectation of these inequalities yields 
$$\OPT \le F(y(\tau - \delta)) + \sum_{v \in S^*} \Delta F(y(\tau -
\delta))$$
from which we easily obtain~\eqnref{eq:6ddd}.

Combining \eqnref{eq:6ddd} with \eqnref{eq:5:line4} yields
$$\Delta(I(\tau)) \ge \delta (1 - d \delta) 
\paren{\OPT - F(y(\tau - \delta)) - \err'_{\tau}} \ge
\delta \paren{\OPT - F(y(\tau - \delta))} - d \delta^2 \OPT -  \delta
\err'_{\tau}$$
which is equivalent to the claimed inequality, since $H(I^-(\tau)) =  F(y(\tau - \delta))$.
\end {proof}

By combining Theorem~\ref{thm:offline_cont_greedy} with Lemma~\ref{lem:cont_greedy_error} we obtain
the following result.

\begin{theorem} \label{thm:full_cont_greedy}
Suppose a noisy version of the \ContinuousGreedy algorithm is run on a matroid
$(\matroidgroundset, \indset)$ of rank $d$ with discretization
parameter $\delta$, such that in stage $\tau$ the selected independent
set $I(\tau)$ satisfies $\charvec{I(\tau)} \cdot \Delta F(y(\tau - \delta))
\ge  \max_{S \in \indset} \charvec{S} \cdot \Delta
F(y(\tau - \delta)) - \err'_{\tau}$ for some error terms
$\set{\err'_{\tau} : \tau \in \delta \nats}$.  Then the random set $S$
output by the algorithm satisfies
$$ \expct{f(S)} \ge \paren{1 - \frac1e - d \delta } \OPT -
\sum_{\tau}  \delta \err'_{\tau}.$$
\end{theorem}

\subsubsection{The Meta-actions Analysis of the Online Continuous
  Greedy Algorithm}

After running the online continuous greedy algorithm for $T$ rounds,
we may view the $T$ choices made by the algorithm $\Alg_{\tau}$ for any fixed stage
$\tau$ in aggregate, as a single choice made in the (noisy) execution
of the offline continuous greedy algorithm in a larger instance.
This larger instance is the direct sum of the $T$ instances.
Thus we have $T$ copies of the input matroid,
$\matroid_t = (\matroidgroundset \times \set{t}, \indset \times \set{t})$ for $1 \le t \le T$,
from which we construct a larger matroid $\matroid^{\oplus}$ with groundset 
$\matroidgroundset^{\oplus} := \bigcup_{t=1}^T\matroidgroundset_t$
and independent sets $\indset^{\oplus} := \set{S : \forall t, S \cap
  \matroidgroundset_t \in \indset_t}$,
and we have $T$ monotone submodular objective functions $f_1, \ldots,
f_T$ from which we construct the monotone submodular objective
$f^{\oplus}(S) := \sum_{t=1}^T f_t(S \cap \matroidgroundset_t)$.
By analogy with the single instance case, we can define the
multilinear extension $F^{\oplus}:[0,1]^{\matroidgroundset \times [T]}
\to \NonNegativeReals$ of $f^{\oplus}$ by 
$F^{\oplus}(Y) = \expct{f^{\oplus}(S_Y)}$, where $S_Y$ is a random set
such that each $v \in \matroidgroundset_t$ is included in $S_Y$
independently with probability $Y_{vt}$.
We can also define an objective on rows, $H^{\oplus}(\Rows) = F^{\oplus}(\sum_{(S, \tau) \in
  \Rows} \delta [S])$ where $[S]$ is the $0/1$ characteristic matrix
of $S$ with $[S]_{vt} = 1$ iff $(v, t) \in S$.
Call a row $S$ \emph{proper} if $S \in \indset^{\oplus}$ and
also, $S$ consists of $T$ copies of the same subset of
$\matroidgroundset$, so that $S = Q \times [T]$ for some $Q \in
\indset$.
We interpret the online problem as noisily 
selecting $1/\delta$ proper rows to maximize $H^{\oplus}$.
By construction, the regret incurred by each $\Alg_{\tau}$ over the $T$
rounds of the online algorithm is precisely the error $\err'_\tau$
measured with respect to this larger instance.  
Combining this fact with Theorem~\ref{thm:full_cont_greedy} yields the
following result.

\begin{theorem} \label{thm:ocg-bound}
  Let $r(\tau)$ be the regret of $\Alg_{\tau}$ in the online continuous
  greedy algorithm with discretization parameter $\delta$.  Then 
  $$ \expct{\sum_{t=1}^T f_t(S_t)} \ge \paren{1- \frac1e - d\delta} \max_{S
    \in \indset} \sum_{t=1}^T f_t(S) - \delta \sum_{s=1}^{1/\delta}
  r(s\delta) $$
\end{theorem}

The regret bound for the follow the perturbed leader algorithm
(Theorem~$1.1$ of~\cite{KV05}), suitably tailored for our purposes, is 
$O(d\sqrt{n g T})$ where $d$  is the rank of $\matroid$, $n =
|\matroidgroundset|$, and $g = \max_{t, v} f_t(\set{v})$ is an
upper bound for each coordinate of $\Delta F_t$.  
Therefore if we were to feed back the exact marginals $\Delta
F_t(y(\tau))$ to each $\Alg_{\tau}$, we would obtain the following bound. 
$$ \expct{\sum_{t=1}^T f_t(S_t)} \ge \paren{1-\frac{1}{e} - d\delta} \max_{S
    \in \indset} \sum_{t=1}^T f_t(S) -  O(d\sqrt{n g T})$$
Though we cannot precisely compute these marginals $\Delta
F_t(y(\tau))$, we will still obtain the above bound using unbiased
estimates of them, using only one evaluation of each function $f_t$.

\subsubsection{Generating Feedback: How to Estimate the Marginal} \label{sec:feedback}

We could follow~\cite{calinescu07} and take
sufficiently many samples to ensure that with high probability the
estimates for $\Delta F(y^t(\tau - \delta))$ are sharp.  Instead, we
will feed back a rough estimate of $\Delta F(y^t(\tau - \delta))$.
Perhaps surprisingly, using unbiased estimates, or even certain biased
estimates, results in exactly the
same expected regret bounds as using the actual marginals.
This is a consequence of the following lemma. \andreas{This lemma is stated in a more complicated way then needed here.}

\begin{lemma}{(Lemma $5$ of \cite{streeter07tr})} \label{lem:magic}
Let $\Alg$ be a no-regret algorithm that incurs a worst-case expected
regret $R(T)$ over $T$ rounds in the full-information feedback model.
For each feasible $S$, let $x^t_S$ be the payoff of $S$ in round $t$, and let $\hat{x}^t_S$
be an estimate of $x^t_S$ such that $\expct{\hat{x}^t_S} = \gamma
x^t_S + \eta^t$ for all $S$.
Let $\Alg'$ be the algorithm that results from running $\Alg$ but
feeding back the estimates $\hat{x}^t_S$ instead of the true payoffs $x^t_S$.
Then the expected worst-case regret of $\Alg'$ is $R(T)/\gamma$.
\end{lemma}

To generate an unbiased estimate of $\Delta F(y)$, from the
multilinearity of $F$ with only one evaluation of $f$, 
we pick some $v \in \matroidgroundset$ uniformly at random, and 
sample $\thresh \in [0,1]^{\matroidgroundset}$ uniformly at random, and 
define $A = \set{u : \thresh_u \le y_u }$, $B = A \cup \set{v}$.  
Then sample $X \in \set{0,1}$ uniformly at random.
We then generate an estimate of the gradient 
$$\est = \left\{ 
\begin{array}{ll} -2|\matroidgroundset| f(A) \cdot \charvec{\set{v}} & \text{if }X = 0 \\
 \ \ \  2|\matroidgroundset| f(B)  \cdot \charvec{\set{v}} & \text{if }X = 1  
\end{array}\right.$$

\begin{lemma}
The estimate $\est$ 
described above is an unbiased estimate of $\Delta F$.
\end{lemma}

\begin{proof}
By the multilinearity of $F$, 
\begin{eqnarray}
  \label{eq:1}
  (\Delta F(y))_v & = & F(y + (1-y_v) \charvec{\set{v}}) - F(y)\\
  & = & \expct{f(S_{y} \cup \set{v}) - f(S_{y})}
\end{eqnarray}
The estimate has $\est_v \neq 0$ only if $v$ is the selected
coordinate, so that $\expct{\est_v} = \expct{\est_v \mid v\text{
    selected}} \Pr{v\text{ selected} } $, and $v$ is selected
with probability $1/|\matroidgroundset|$.
Hence $\expct{\est_v} = \frac{1}{|\matroidgroundset|}\expct{\est_v
  \mid v\text{ selected}} $.
In the case that $v$ is selected, it is clear that $f(A)$ is an
unbiased estimate of $F(y)$ 
and $f(B)$ is an unbiased estimate of $F(y +(1-y_v) \charvec{\set{v}})$.  
Hence,
$\expct{\est_v  \mid v\text{ selected}, X = 0} = - 2
|\matroidgroundset| F(y)$ and
$\expct{\est_v  \mid v\text{ selected}, X = 1} = 2 |\matroidgroundset| F(y
+ (1-y_v) \charvec{\set{v}})$. 
It follows that 
$\expct{\est_v  \mid v\text{ selected}} = |\matroidgroundset| (\Delta F(y))_v$,
 and 
$\expct{\est_v} = (\Delta F(y))_v$.
\end{proof}

Of course, one may take more samples to reduce the variance in the
estimates for $\Delta F$, e.g., several per coordinate, 
though Lemma~\ref{lem:magic} indicates that
this will not improve the worst-case regret.

\subsubsection{Using Online Continuous Greedy for the Offline Problem}
\andreas{Not clear if this discussion is needed, unless the online algorithm is more efficient than the offline algorithm.}
The \OnlineContinuousGreedy algorithm guarantees no-regret against
arbitrary sequences of objectives $\set{f_t}_{t \ge 0}$.  Our analysis
also suggests it can be used to solve the original offline problem, 
$\max_{S \in \indset} f(S)$, faster than \vondrak's original (offline)
\ContinuousGreedy algorithm.
As described in~\citep{calinescu07}, \ContinuousGreedy has a running
time of order $\tilde{O}(n^8)$ plus the time required to evaluate $f$ on $\tilde{O}(n^7)$
arguments, where the $\tilde{O}$ notation suppresses logarithmic
factors.  
\citet{calinescu07} describe this high complexity as being
``mostly due to the number of random samples necessary to achieve high
probability bounds'' and suggest that this ``can be substantially
reduced by a more careful implementation and analysis.''

To use \OnlineContinuousGreedy to solve the offline problem, 
simply run it for some number $T$ of rounds with objectives 
$f_t \equiv f$ for all rounds $t$, generating solutions
$\set{y^t}_{t=1}^{T}$ to the fractional relaxation of the problem,
while omitting the step in which the algorithm generates 
$S_t$ via the pipage rounding of $y^t$.
Then select a round $t \in [T]$ uniformly at random, and 
pipage round $y^t$
to obtain a set $S$.
As $T$ increases the bound on the expected quality of the solution will increase.
Fix a desired relative error bound
$\err$, so that we will compute a set $S'$ such that $\expct{f(S')} \ge (1 - 1/e
- \err)\max_{S \in \indset} f(S)$.
Setting $\delta = \err/2d$ and $T = \frac{4d^2 n g}{\err^2 \OPT^2}$
suffices, where 
$d = \max \set{|S| : S \in \indset}$ is the rank of the input matriod
$\matroid = (\matroidgroundset, \indset)$, $n = |\matroidgroundset|$
is the size of the groundset, $\OPT = \max_{S \in \indset} f(S)$ is
the offline optimum value, and $g = \max_{v \in \matroidgroundset}
f(\set{v})$ is an upper bound on $(\Delta F(y))_v$ for all $y \in \polytope{\matroid}$.
The running time is proportional to $nT/\delta = O\paren{\frac{d^3 n^2
  g}{\err^3 \OPT^2}}$ plus the time for the final pipage rounding
operation and the time required to evaluate $f$ on $T / \delta = O\paren{\frac{d^3 n
  g}{\err^3 \OPT^2}}$
arguments.

\section{Evaluation} \label {sec:evaluation}
We evaluate \OnlineGreedy experimentally on two applications: Learning to rank blogs that are effective in detecting cascades of information, and allocating advertisements to maximize revenue.

\vspace{\presubsec}
\subsection{Online learning of diverse blog rankings} %
\vspace{\postsubsec}
We consider the problem of ranking a set of blogs and news sources on the web.
Our approach is based on the following idea: A blogger writes a posting, and, after some time, other postings link to it, forming cascades of information propagating through the network of blogs.

More formally, an information cascade is a directed acyclic graph of vertices
(each vertex corresponds to a posting at some blog), where edges are
annotated by the time difference between the postings.
Based on this
notion of an information cascade, we would like to select blogs
that detect big cascades (containing many nodes) as early as
possible (i.e., we want to learn about an important event before
most other readers). In
\cite{leskovec07}
 it is shown how one
can formalize this notion of utility using a monotone submodular function
that measures the informativeness of a subset of blogs. Optimizing the submodular function yields a small set of blogs that ``covers'' most cascades.  This utility function prefers \emph{diverse} sets of blogs, minimizing the overlap of the detected cascades, and therefore minimizing redundancy.

The work by \cite{leskovec07} leaves two major shortcomings: Firstly, they select a \emph{set} of blogs rather than a \emph{ranking}, which is of practical importance for the presentation on a web service. Secondly, they do not address the problem of \emph{sequential} prediction, where the set of blogs must be updated dynamically over time. In this paper, we address these shortcomings.

\vspace{\prepara}
\paragraph{Results on offline blog ranking.}
\newcommand {\Blogs} {\mathcal{B}}
In order to model the blog ranking problem, we adopt the assumption that different users have different attention spans: Each user will only consider blogs appearing in a particular subset of positions. In our experiments, we assume that the probability that a user is willing to look at position $\k$ is proportional to $\gamma^\k$, for some discount factor $0<\gamma<1$.
More formally, let $g$ be the monotone submodular function measuring the informativeness of any set of blogs, defined as in \cite{leskovec07}.
Let $\Partition_\k = \Blogs \times \set {\k}$, where $\Blogs$ is the set of blogs.  Given an assignment $S \in \Feasible$, let $S^{[\k]} = S \cap \set { \Partition_1 \cup \Partition_2 \cup \ldots \cup \Partition_\k }$ be the assignment of blogs to positions 1 through $\k$.
We define the \emph{discounted} value of the assignment $S$ as
$
f(S) = \sum_{\k=1}^\K \gamma^\k \paren{ g(S^{[\k]})-g(S^{[\k-1]}) }.
$
It can be seen that $f: 2^\groundset \to \NonNegativeReals$ is monotone submodular.

For our experiments, we use the data set of \cite{leskovec07}, consisting of 45,192 blogs, 16,551 cascades, and 2 million postings collected during 12 months of 2006. We use the \emph{population affected} objective of \cite{leskovec07}, and use a discount factor of $\gamma = 0.8$.
Based on this data, we run our \OfflineGreedy algorithm with varying numbers of colors $\NumColors$ on the blog data set. \figref{fig:blogs_offline} presents the results of this experiment. For each value of $\NumColors$, we generate 200 rankings, and report both the average performance and the maximum performance over the 200 trials. Increasing $\NumColors$ leads to an improved performance over the locally greedy algorithm ($\NumColors=1$).

\vspace{\prepara}
\paragraph{Results on online learning of blog rankings.}

We now consider the online problem where on each round $t$ we want to output a ranking $S_t$. After we select the ranking, a new set of cascades occurs, modeled using a separate submodular function $f_t$, and we obtain a reward of $f_t(S_t)$. In our experiments, we choose one assignment per day, and define $f_t$ as the utility associated with the cascades occurring on that day. Note that $f_t$ allows us to evaluate the performance of any possible ranking $S_t$, hence we can apply \OnlineGreedy in the \emph{full-information feedback} model.

We compare the performance of our online algorithm using $\NumColors =
1$ and $\NumColors = 4$. \figref{fig:blogs_discount} presents the
average cumulative reward gained over time by both algorithms. We
normalize the average reward by the utility achieved by
the \OfflineGreedy algorithm (with $\NumColors = 1$) applied to the
entire data set.  \figref{fig:blogs_discount} shows that the
performance of both algorithms rapidly (within the first
47 rounds) converges to the performance of the offline
algorithm. The \OnlineGreedy algorithm with $\NumColors = 4$ levels
out at an approximately 4\% higher reward than the
algorithm with $\NumColors=1$.

\begin{figure}[t]
\centering
\subfigure[Blogs: Offline results]{
\includegraphics[width=.31\textwidth]{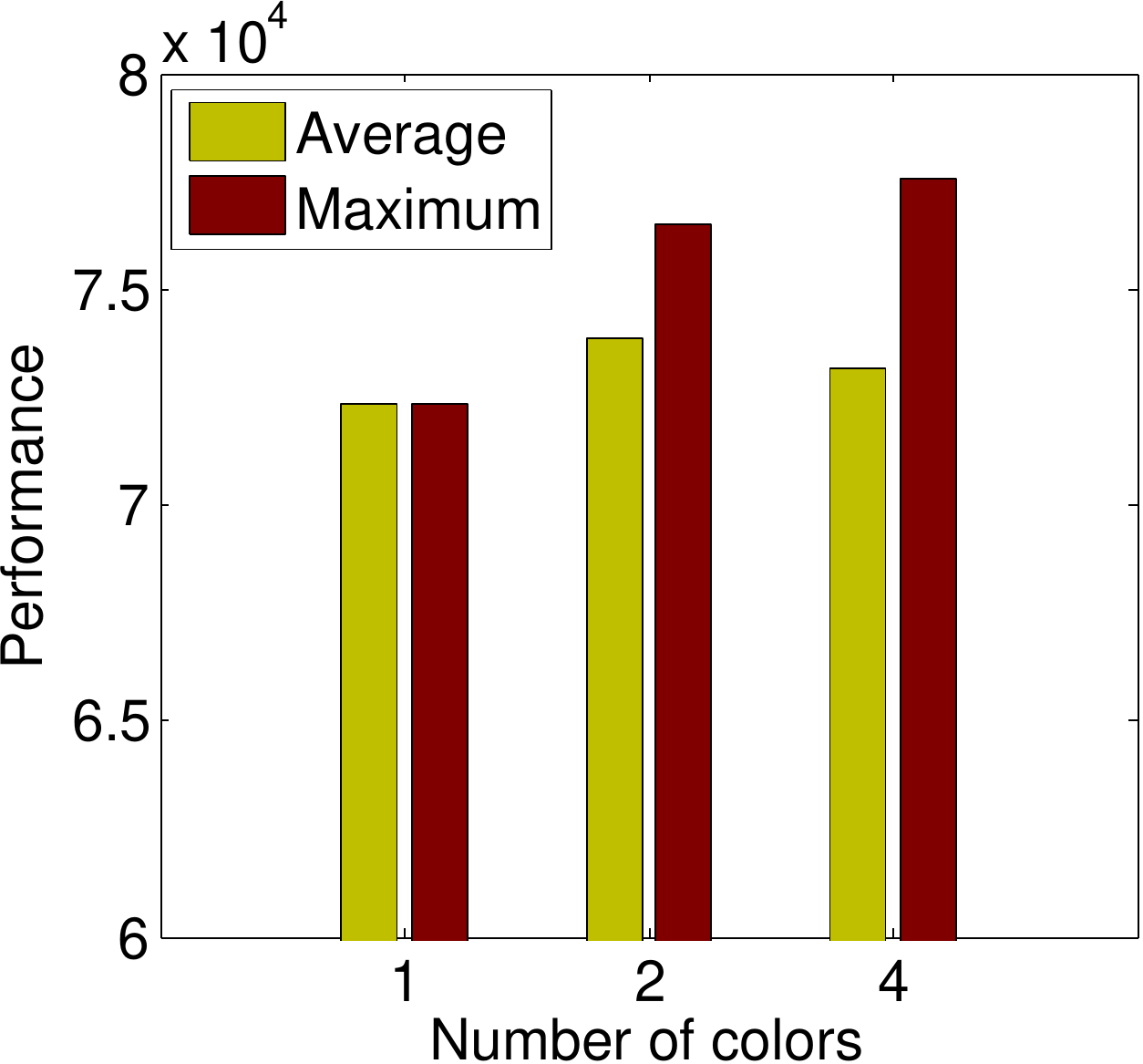}
 \label{fig:blogs_offline}
 }
 \subfigure[Blogs: Online results]{
\includegraphics[width=.31\textwidth]{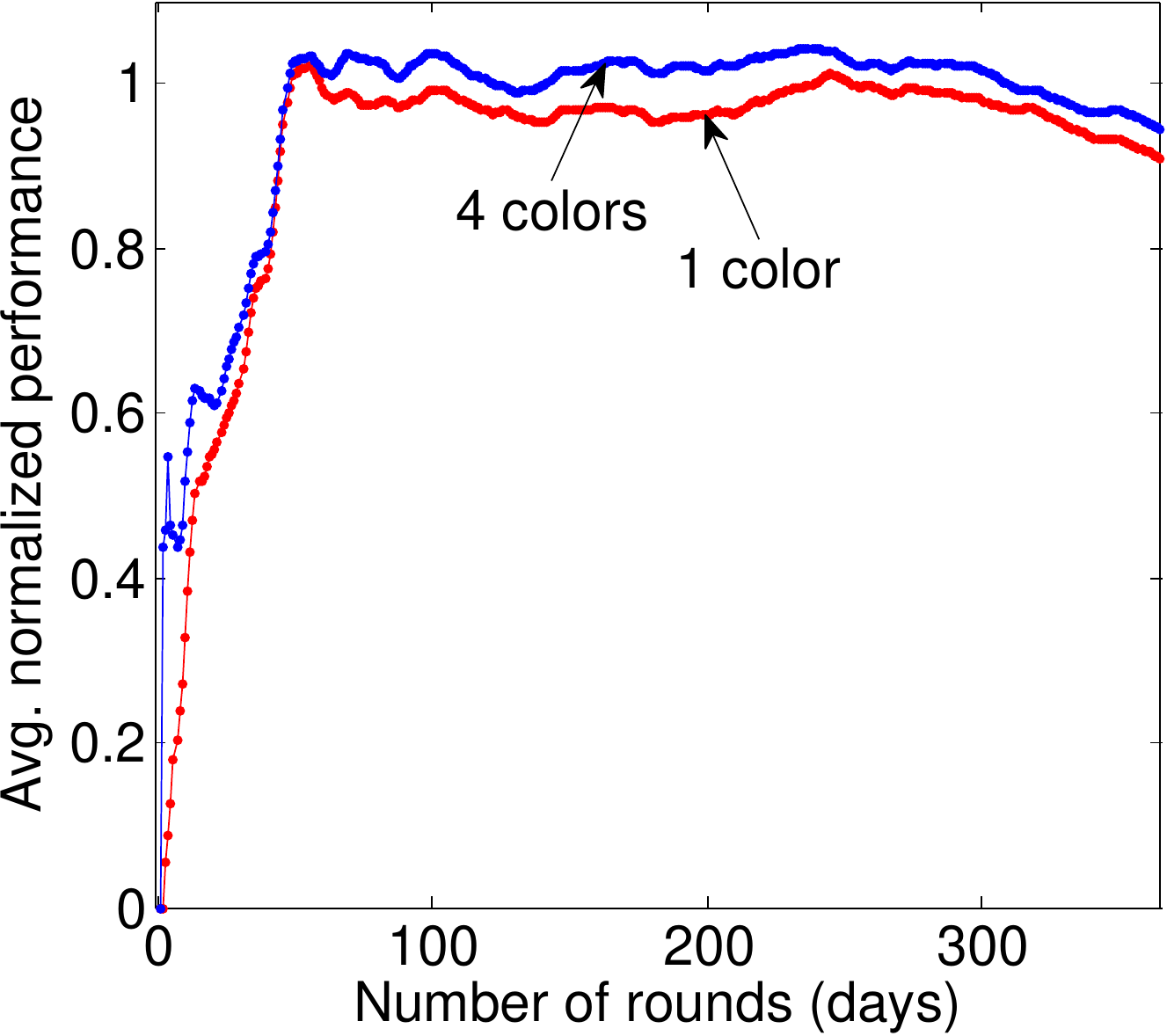}
 \label{fig:blogs_discount}
 }
 \subfigure[Ad display: Online results]{
\includegraphics[width=.30\textwidth]{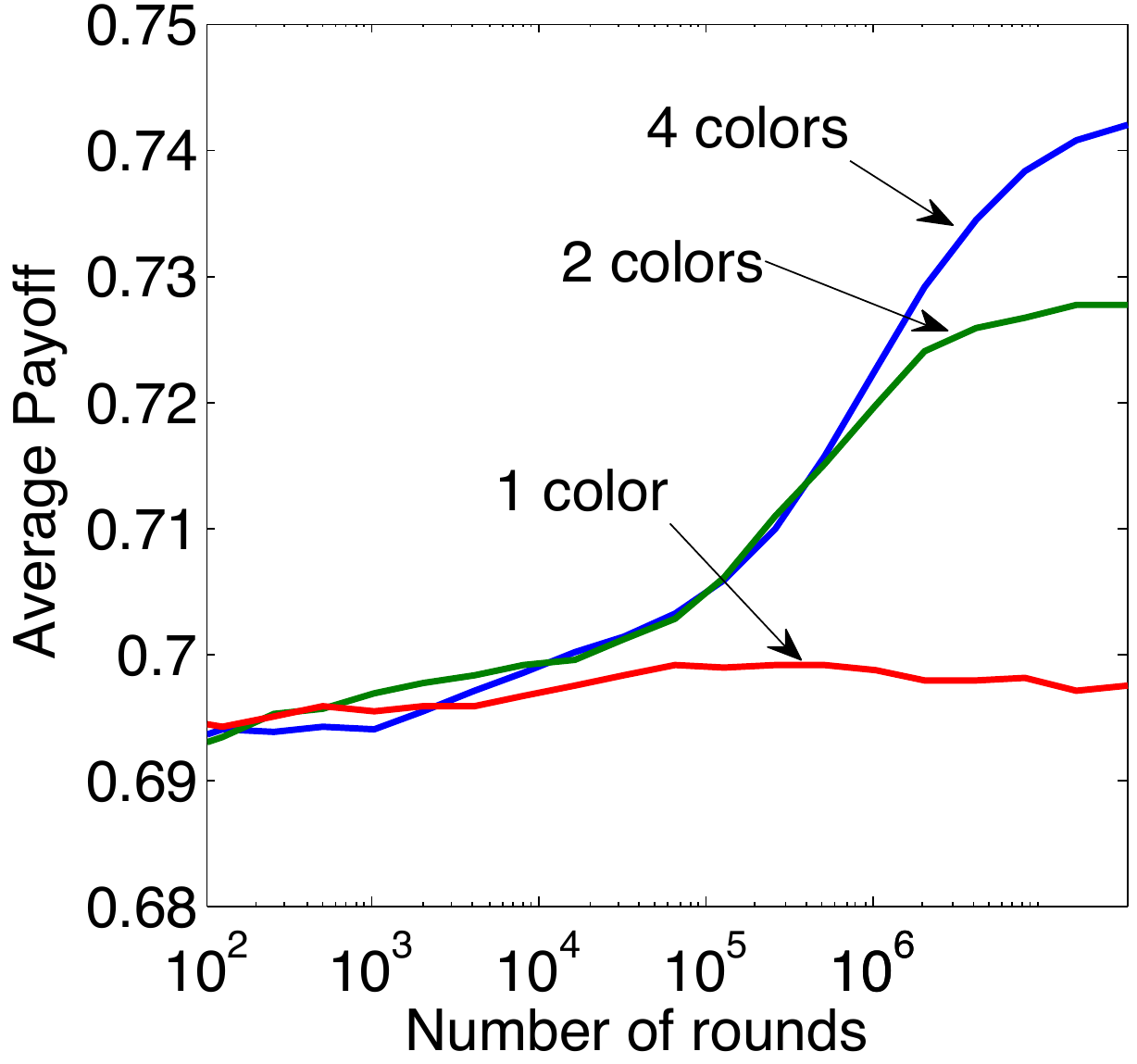}
\label{fig:online_ads}
 }
\caption{\small (a,b) Results for discounted blog ranking ($\gamma = 0.8$), in offline (a) and online (b) setting. (c) Performance of \OnlineGreedy with $\J = 1$,
$2$, and $4$ colors for the sponsored search ad selection problem (each round is a query).  Note that $\J = 1$ corresponds to the online
algorithm of~\cite{radlinski08} and \cite{streeter08}.\vspace{-4mm}}
\end{figure}
\vspace{\presubsec}
\subsection{Online ad display} \label{ssec:online_ads}
\vspace{\postsubsec}
\newcommand {\Ads} {\mathcal{A}}
\newcommand {\Pclick}{\ensuremath{p_{\text{click}}}}
\newcommand {\Pabandon}{\ensuremath{p_{\text{abandon}}}}

We evaluate \OnlineGreedy for the sponsored search ad selection
problem in a simple Markovian model incorporating the value of diverse results and
complex position-dependence among clicks.
In this model, each user $u$ is defined by
two sets of probabilities:
$\Pclick(a)$ for each ad $a \in \Ads$, and $\Pabandon(\k)$ for each
position $\k \in \bracket{\K}$.  When presented an assignment of ads
$\set{a_1, a_2, \ldots, a_\K}$, where $a_\k$ occupies position $\k$,
the user scans the positions in increasing order.
For each position $\k$, the user clicks on $a_\k$ with probability
$\Pclick(a_\k)$, leaving the results page forever.
Otherwise, with probability $(1 - \Pclick(a_\k)) \cdot \Pabandon(\k)$, the user loses interest and
abandons the results without clicking on anything.
Finally, with probability $(1 - \Pclick(a_\k)) \cdot (1 - \Pabandon(\k))$, the user proceeds to look
at position $\k+1$.
The reward function $f_t$ is the number of clicks, which is either zero or
one. We only receive information about $f_t(S_t)$ (i.e., \emph{bandit feedback}).

In our evaluation, there are $5$ positions, $20$ available ads,
and two (equally frequent) types of users: type $1$ users interested in
all positions ($\Pabandon \equiv 0$), and type $2$ users that quickly
lose interest ($\Pabandon \equiv 0.5$).  There are also two types of
ads, half of type $1$ and half of type $2$, and 
users are probabilistically more interested in ads of their own type
than those of the opposite type.  Specifically, 
for both types of users we set $\Pclick(a) = 0.5$ if $a$ has the
same type as the user, and $\Pclick(a) = 0.2$ otherwise.
In~\figref{fig:online_ads} we compare the performance of \OnlineGreedy with $\J = 4$ to the online
algorithm of~\cite {radlinski08, streeter08}, based on the average of $100$ experiments.  The latter algorithm is
equivalent to running \OnlineGreedy with $\J = 1$.  They perform
similarly in the first $10^4$ rounds; thereafter the former algorithm dominates.

It can be shown that with several different types of users with
distinct $\Pclick(\cdot)$ functions the
offline problem of finding an assignment within
$1 - \frac 1 e + \eps$ of optimal is $\NP$-hard.
This is in contrast to the
case in which $\Pclick$ and $\Pabandon$ are the \emph {same} for all
users; in this case the offline problem simply requires finding an
optimal policy for a Markov decision process, which can be done
efficiently using well-known algorithms.
A slightly different Markov model of user behavior which is efficiently solvable was considered in~\cite{aggarwal08}.
In that model, $\Pclick$ and $\Pabandon$ are the same for all users, and
$\Pabandon$ is a function of the ad in the slot currently being
scanned rather than its index.

\section{Related Work} \label {sec:related_work}
\andreas{Need some clarification on how this paper relates to our
  previous NIPS paper.} 

An earlier version of this work appeared as~\cite{streeter09}
(also~\cite{arxiv-version}).  The present article is significantly extended, including a new algorithm for online optimization over arbitrary matroids. For a general introduction to the
literature on submodular function maximization (including offline
versions of the problems we study), see~\cite{vondrak07} and the survey by \citet{krause14survey}.
For an overview of applications of submodularity to machine learning and artificial
intelligence, see~\cite{krause11submodularity}.

In the online setting, the most closely related work is that of~\cite{streeter08}.  Like us, they consider
sequences of monotone submodular reward functions that arrive online,
and develop an online algorithm that uses multi-armed bandit
algorithms as subroutines.  The key difference from our work is that,
as in \cite {radlinski08}, they are concerned with selecting a \emph
{set} of $\NumPartitions$ items rather than the more general problem
of selecting an \emph{assignment} of items to positions addressed in
this paper.  \citet{Kakade07} considered the general
problem of using $\alpha$-approximation algorithms to construct 
no $\alpha$-regret online algorithms, and essentially proved it could 
be done for the class of linear optimization problems in which the cost function
has the form $c(S, w)$ for a solution $S$ and weight vector $w$, and 
$c(S, w)$ is linear in $w$.   However, their result is orthogonal to
ours, because our objective function is submodular and not
linear\footnote{
\ifthenelse{\boolean{istechrpt}}{
Of course, it is possible to linearize a submodular function by
using a separate dimension for every possible function argument, 
but this results in an exponential number of dimensions, which leads
to exponentially worse convergence time and regret bounds for the
algorithms in~\cite{Kakade07} relative to \OnlineGreedy.
}
{One may linearize a submodular function by using a separate dimension for every possible function argument, 
but this leads to exponentially worse convergence time and regret bounds for the
algorithms in~\cite{Kakade07} relative to \OnlineGreedy.
}
}. %

Since the earlier version of this paper appeared, 
subsequent research has produced algorithms that incorporate \emph{context} into
their decisions:  \cite{dey2013} developed an
algorithm for contextual optimization of sequences of
actions, and used it to optimize control libraries used for various
robotic planning tasks.  \cite{ross2013learning} developed an improved
variant, and applied it to tasks such as news recommendation and
document summarization.

\daniel{Other, more recent papers to consider mentioning here: ``Contextual
  sequence prediction with application to control library
  optimization'', ``Predicting contextual sequences via submodular
  function maximization'', ``Knapsack Constrained Contextual
  Submodular List Prediction with Application to Multi-document
  Summarization'', 
``Diversity maximization under matroid
  constraints'', 
``Combinatorial Multi-Armed Bandit: General
  Framework, Results and Applications'', 
``Learning policies for contextual submodular prediction''}

\vspace{\presec}
\section{Conclusions}
\vspace{\postsec}
\label{sec:conclusions}
In this paper, we showed that important problems, such as ad display in
sponsored search and computing diverse rankings of information sources
on the web, require optimizing assignments under submodular utility
functions. We developed an efficient algorithm, \OfflineGreedy, which
obtains the optimal approximation ratio of $(1-1/e)$ for
this NP-hard optimization problem.  We also developed an online
algorithm, \OnlineGreedy, that asymptotically achieves no
$(1-1/e)$-regret for the problem of repeatedly selecting informative
assignments, under the full-information and bandit-feedback
settings.  We demonstrated that our algorithm outperforms 
previous work on two real world problems, namely online ranking of
informative blogs and ad allocation.  Finally, we developed
\OnlineContinuousGreedy, an online algorithm that can handle more
general matroid constraints, while still guaranteeing no $(1-1/e)$-regret.

\small{
\paragraph{Acknowledgments.}
This work was supported in part by Microsoft Corporation through a
gift as well as through the Center for Computational Thinking at Carnegie
Mellon, by NSF ITR grant CCR-0122581 (The Aladdin Center), NSF grant
IIS-0953413, and by ONR grant N00014-09-1-1044.
}

\bibliography{jmlr}

\ifthenelse{\boolean{istechrpt}}{
\section*{Appendix A: Proofs}

Lemma \ref {lem:locally_greedy} is a corollary of the following more general lemma.  The difference between the two lemmas is that, unlike Lemma \ref {lem:locally_greedy}, Lemma \ref {lem:locally_greedy_error} allows for the possibility that each $\argmax$ in the locally greedy algorithm is evaluated with additive error.  We will need this result in analyzing \OnlineGreedy later on.

\begin {lemma} \label {lem:locally_greedy_error}
Let $f : \Feasible \to \NonNegativeReals$ be a function of the
form $f(\assignment) = f_0(\assignment) + \sum_{\k=1}^\K f_\k(\assignment \cap P_\k)$, where $f_0 :
2^\groundset \to \NonNegativeReals$ is monotone submodular, and
$f_\k : P_\k \to \NonNegativeReals$ is arbitrary for $\k \ge
1$.  Let $L = \set {\ell_1, \ell_2, \ldots, \ell_\K}$, where $\ell_\k
\in P_\k$ (for $1 \le \k \le \K$).   Suppose that for any $\k$,
\begin {equation} \label {eq:near_locally_greedy}
  f(\set {\ell_1, \ell_2, \ldots, \ell_\k})
  \ge \max_{x \in P_\k} \set { f(\set {\ell_1, \ell_2, \ldots,
  \ell_{\k-1}} + x) }  - \eps_\k \mbox { .}
\end {equation}
Then
\[
    f(L) + f_0(L) \ge \max_{\assignment \in \Feasible} \set { f(\assignment) } - \sum_{\k=1}^\K \eps_\k \mbox { .}
\]
\end {lemma}
\begin {proof}
Let $\OPT = \argmax_{\assignment \in \Feasible} \set { f(\assignment) }$, and let $\OPT = \set
{o_1, o_2, \ldots, o_\K}$, where $o_\k \in P_\k$ (for $1 \le \k \le \K$).  Define
\[
    \Delta_\k(L) := f_\k(\set{o_\k}) + f_0(L \cup \set { o_1, o_2,
    \ldots, o_\k })
- f_0(L \cup \set { o_1, o_2, \ldots, o_{\k-1} }) \mbox { .}
\]
Let $L_\k = \set {\ell_1, \ell_2, \ldots, \ell_{\k-1}}$.
Using submodularity of $f_0$, we have
\begin {align*}
\Delta_\k(L) & \le  f_\k(\set{o_\k}) + f_0(L_\k + o_\k) - f_0(L_\k) \\
    & = f(L_\k + o_\k) - f(L_\k) \\
    & \le f(L_\k + \ell_\k) - f(L_\k) + \eps_\k  \mbox { .}
\end {align*}
Then, using monotonicity of $f_0$,
\begin {align*}
f(\OPT)
    & \le f_0(L \cup \OPT) + \sum_{\k=1}^\K f_\k(\OPT \cap P_\k) \\
    & = f_0(L) + \sum_{\k=1}^\K \Delta_\k(L) \\
    & \le f_0(L) + \sum_{\k=1}^\K f(L_\k + \ell_\k) - f(L_\k) + \eps_\k \\
    & = f_0(L) + f(L) - f(\emptyset) + \sum_{\k=1}^\K \eps_\k \mbox { .}
\end {align*}
Rearranging this inequality and using $f(\emptyset) \ge 0$ completes the proof.
\end {proof}

To analyze \OnlineGreedy, we will also need Theorem \ref {thm:offline_greedy_error}, which is a generalization of Theorem \ref {thm:offline_greedy}.

\begin {theorem} \label {thm:offline_greedy_error}
Suppose $f$ is monotone submodular.
Let $G = \set {g_{\k,\j}: \k \in \bracket{\K}, \j \in \bracket {\J} }$,
where $g_{\k,\j} \in \P_\k$ for all $\k$ and $\j$.
Suppose that for all $\k \in \bracket{\K}$ and $\j \in \bracket{\J}$,
\begin {equation} \label {eq:epsilon}
  F(G^-_{\k,\j} + g_{\k,\j}) \ge
  \max_{x \in \P_\k \times \set{\j}} \set { G^-_{\k,\j} + x  } - \eps_{\k,\j}
\end {equation}
where $G^-_{\k,\j} = \set { g_{\k', \j'}: \k' \in \bracket{\NumPartitions}, \j' < \j } \cup \set {
g_{\k',\j}: \k' < \k }$ (i.e., $G^-_{\k,\j}$ equals $G$ just
before $g_{\k,\j}$ is added).
Then
$  f(G) \ge \beta(\J, \K) \cdot \max_{\assignment \in \Feasible} \set {f(\assignment)} -
  \sum_{\k=1}^\K \sum_{\j=1}^\J \eps_{\k,\j}$,
where $\beta(\J, \K)$ is defined as $1 - (1 - \frac {1} {\J})^\J - {\K \choose 2} \J^{-1} $.
\end {theorem}
\begin {proof}
The proof is identical to the proof of Theorem \ref {thm:offline_greedy} in the main text, using Lemma \ref {lem:locally_greedy_error} in place of Lemma \ref {lem:locally_greedy}.
\end {proof}

\noindent Finally, we prove Theorem \ref {thm:online_greedy}, which we restate here for convenience.\\

\noindent \textbf{Theorem \ref{thm:online_greedy}} 
\emph{
Let $\Regret_{\k,\j}$ be the regret of $\Alg_{\k,\j}$, and let $\beta(\K,
\J) = 1 - \paren { 1 - \frac 1 \J }^\J - {\K \choose 2} \J^{-1}$.
Then
\[
    \E { \sum_{t=1}^T f_t(G_t) } \ge \beta(\K, \J) \cdot \max_{\assignment \in \Feasible} \set { \sum_{t=1}^T f_t(\assignment) } - \E { \sum_{\k=1}^\K \sum_{\j=1}^\J \Regret_{\k,\j} } \mbox { .}
\]
}
\begin {proof}
The idea of the proof is to view \OnlineGreedy as a version of
\OfflineGreedy that, instead of greedily selecting single (element,color) pairs
$g_{\k,\j} \in \P_\k \times
\set {\j}$, greedily selects (element vector, color) pairs
$\vec g_{\k,\j} \in \P_\k^T \times \set {\j}$, where $T$ is the number of rounds.

First note that for any $\k \in \bracket{\K}$, $\j \in \bracket {\J}$,
and $x \in \P_\k \times \set{\j}$,
by definition of $\Regret_{\k,\j}$ we have
\[
  \sum_{t=1}^T \bar F_t\paren{G^t_{\k,\j} + g^t_{\k,\j}} \ge
  \paren{\sum_{t=1}^T \bar F_t\paren{G^t_{\k,\j} + x}} -\Regret_{\k,\j} \mbox { .}
\]
Taking the expectation of both sides over $\Color$, and choosing $x$
to maximize the right hand side, we get
\begin {equation} \label {eq:expected_regret}
  \sum_{t=1}^T F_t\paren{G^t_{\k,\j} + g^t_{\k,\j}} \ge
  \max_{x \in \P_\k \times \set{\j}} \set { \sum_{t=1}^T F_t\paren{G^t_{\k,\j} + x} } - \eps_{\k,\j}
\end {equation}
where we define $F_t(\assignment) = \Esub{\cvec}{f_t(\sample_{\cvec}(\assignment))}$ and
$\eps_{\k,\j} = \E {\Regret_{\k,\j}}$.

We now define some additional notation.
For any set $\vec S$ of vectors in $\groundset^T$, define
\[
  f(\vec S) = \sum_{t=1}^T f_t\paren{\set{\vec a_t: \vec a \in \vec
  S}} \mbox { .}
\]
Next, for any set $\vec S$ of (element vector, color) pairs in
$\bigcup_{\k=1}^\K (\P_\k^T
\times \set {\j})$, define $\sample_{\cvec}(\vec S) = \bigcup_{\k=1}^\K
\set { \vec x \in \P_\k^T: (\vec x, \Color_\k) \in \vec S  }$.
Define $F(\vec S) = \Esub{\cvec}{f(\sample_{\cvec}(\vec S))}$.
By linearity of expectation,
\begin {equation} \label {eq:linexp}
  F(\vec S) = \sum_{t=1}^T F_t\paren{\set { (\vec x_t, \j): (\vec x, \j)
  \in \vec S  }} \mbox { .}
\end {equation}
Let $\vec g_{\k,\j} = (\vec x, \j)$, where $\vec x$ is such that $(\vec
x_t, \j) = g^t_{\k,\j}$ for all $t \in \bracket{T}$.
Analogously to $G^{t-}_{\k,\j}$, define
$\vec G^-_{\k,\j} = \set {\vec g_{\k',
\j'}: \k' \in \bracket{\K}, \j' < \j} \cup \set {\vec g_{\k',\j}: \k' < \k }$.
By \eqref {eq:linexp}, for any $(\vec x, \j) \in \groundset^T \times \bracket{\J}$ we have
$F(\vec G^-_{\k,\j} + (\vec x, \j)) = \sum_{t=1}^T F_t(G^{t-}_{\k,\j} +
(\vec x_t, \j))$.
Combining this with \eqref {eq:expected_regret}, we get
\begin {equation} \label {eq:goal}
  F\paren{\vec G^-_{\k,\j} + \vec g_{\k,\j}} \ge \max_{a: a \in \P_\k} \set {
  F\paren{\vec G^-_{\k,\j} + (a^T, \j)} } - \eps_{\k,\j}
\end {equation}
where $a^T$ is the unique element of $\set{a}^T$.
Having proved \eqref {eq:goal}, we can now use Theorem \ref
{thm:offline_greedy_error} to complete the proof.
Let $\vec P_\k := \set {a^T: a \in P_\k}$ for each $\k \in \bracket{\K}$,
and define a new partition matroid over ground set $\set{a^T : a \in \groundset}$
with feasible solutions
$\vec \Feasible := \set{ \vec \assignment : \forall
\k \in \bracket{\K}, |\vec \assignment \cap \vec \Partition_\k| \le 1}$.
Let $\vec G = \set {\vec g_{\k,\j}: \k \in \bracket{\K}, \j \in \bracket{\J}}$.
As argued in the proof of Lemma \ref {lem:outer_loop},
$F_t$ is monotone submodular.  Using this fact together with \eqref
{eq:linexp}, it is straightforward to show that $F$ itself is monotone
submodular.  Thus by Theorem \ref {thm:offline_greedy_error},
\[
  F(\vec G) \ge \beta(\J,\K) \cdot \max_{\vec \assignment \in \vec \Feasible} \set {
  f(\vec \assignment) } - \sum_{\k=1}^\K \sum_{\j=1}^\J \eps_{\k,\j} \mbox { .}
\]
To complete the proof, it suffices to show that
$F(\vec G) = \E {\sum_{t=1}^T f_t(G_t) }$, and that
$\max_{\vec \assignment \in \vec \Feasible} \set {f(\vec \assignment)} \ge \max_{\assignment \in \Feasible}
\set { \sum_{t=1}^T f_t(\assignment) }$.  Both facts follow easily from
the definitions.
\end {proof}

} {}

\end{document}